\documentclass[accepted]{uai2025} 
                        
\usepackage[american]{babel}

\usepackage{natbib} 
\usepackage{mathtools} 
\usepackage{booktabs} 
\usepackage{tikz} 

\usepackage{graphicx}
\usepackage{amsfonts,enumitem,amsmath,amsthm}

\usepackage{microtype}      
\usepackage{xcolor}         
\usepackage{float}          
\usepackage{subfigure}      
\usepackage{multirow}

\usepackage{graphicx,appendix,amssymb}
\usepackage[normalem]{ulem}
\usepackage{soul}
\usepackage{makecell}
\newtheorem{theorem}{Theorem}
\newtheorem{definition}{Definition}
\newtheorem{lemma}{Lemma}

\newtheorem{corollary}{Corollary}
\newtheorem{remark}{Remark}

\title{Do Vendi Scores Converge with Finite Samples?\\
Truncated~Vendi Score for Finite-Sample Convergence Guarantees 
}

\author[1]{\href{mailto:<aospanov9@cse.cuhk.edu.hk>?Subject=Your UAI 2025 paper}{Azim Ospanov}{}}
\author[1]{\href{mailto:<farnia@cse.cuhk.edu.hk>?Subject=Your UAI 2025 paper}{Farzan Farnia}{}}
\affil[1]{%
    Department of Computer Science and Engineering, 
    The Chinese University of Hong Kong}
    
\begin{document}
\maketitle

\begin{abstract}
    Evaluating the diversity of generative models without reference data poses methodological challenges. The reference-free Vendi \citep{friedman_vendi_2023} and RKE \citep{jalali_information-theoretic_2023} scores address this by quantifying the diversity of generated data using matrix-based entropy measures. Among these two, the Vendi score is typically computed via the eigendecomposition of an $n \times n$ kernel matrix constructed from n generated samples. However, the prohibitive computational cost of eigendecomposition for large $n$ often limits the number of samples used to fewer than 20,000. In this paper, we investigate the statistical convergence of the Vendi and RKE scores under restricted sample sizes. We numerically demonstrate that, in general, the Vendi score computed with standard sample sizes below 20,000 may not converge to its asymptotic value under infinite sampling. To address this, we introduce the \emph{$t$-truncated Vendi score} by truncating the eigenspectrum of the kernel matrix, which is provably guaranteed to converge to its population limit with $n=\mathcal{O}(t)$ samples. We further show that existing Nyström and FKEA approximation methods converge to the asymptotic limit of the truncated Vendi score. In contrast to the Vendi score, we prove that the RKE score enjoys universal convergence guarantees across all kernel functions. We conduct several numerical experiments to illustrate the concentration of Nyström and FKEA computed Vendi scores around the truncated Vendi score, and we analyze how the truncated Vendi and RKE scores correlate with the diversity of image and text data. The code is available at \url{https://github.com/aziksh-ospanov/truncated-vendi}.
\end{abstract}

\section{Introduction}
\begin{figure*}[h]
    \centering
    \includegraphics[width=0.89\linewidth]{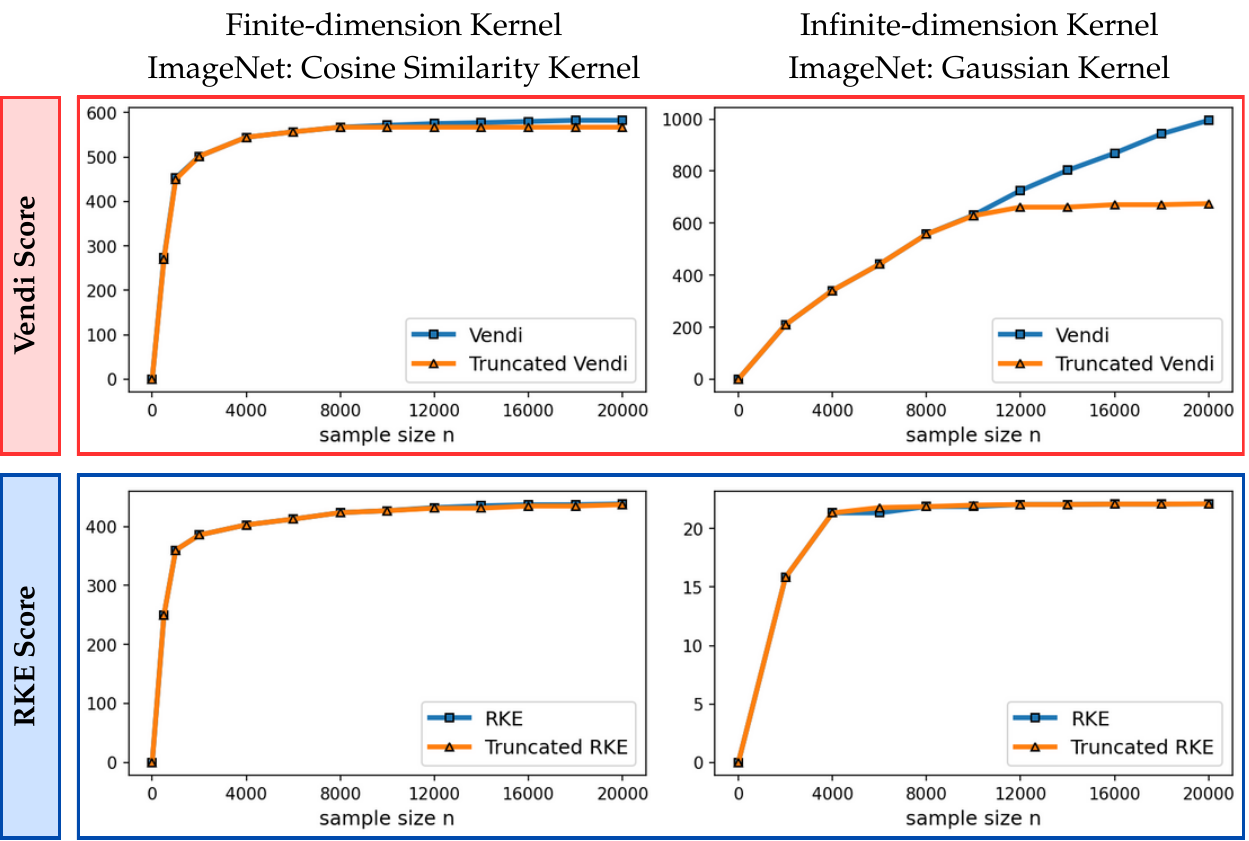}
    \caption{Statistical convergence of Vendi and RKE scores for different sample sizes on ImageNet data: (Left plots) finite-dimension cosine similarity kernel (Right plots) infinite dimension Gaussian kernel with bandwidth $\sigma=30$.  
    The RKE and truncated Vendi scores converged with below 20000 samples, but the Vendi score with Gaussian kernel did not converge.}
    \vspace{-5mm}
  \label{fig:kernel_convergence}
\end{figure*}

The increasing use of generative artificial intelligence has underscored the need for accurate evaluation of generative models.  In practice, users often have access to multiple generative models trained with different training datasets and algorithms, requiring evaluation methods to identify the most suitable model. The feasibility of a model evaluation approach depends on factors such as the required generated sample size, computational cost, and the availability of reference data. Recent studies on evaluating generative models have introduced assessment methods that relax the requirements on data and computational resources. 

Specifically, to enable the evaluation of generative models without reference data, the recent literature has focused on reference-free evaluation scores that remain applicable in the absence of reference samples. The Vendi score \citep{friedman_vendi_2023} is one such reference-free metric that quantifies the diversity of generated data using the entropy of a kernel similarity matrix formulated for the generated samples. Given the  sorted eigenvalues $\lambda_1\ge\cdots \ge\lambda_n$ of the normalized matrix $\frac{1}{n}K$\footnote{In general, we consider the trace-normalized kernel matrix $\frac{1}{\mathrm{Tr}(K)}K$, which given $\forall x: \: k(x,x)=1$, reduces to $\frac{1}{n}K$.}  for the kernel similarity matrix $K=\bigl[k(x_i,x_j)\bigr]_{1\le i,j\le n}$ of $n$ generated samples $x_1,\ldots , x_n$, the definition of (order-1) Vendi score is as: 
\begin{equation}\label{Eq: Intro-Vendi Score}
    \mathrm{Vendi}(x_1,\ldots ,x_n) \, := \, \exp\Bigl(\,\sum_{i=1}^n \lambda_i \log\frac{1}{\lambda_i}\,\Bigr)
\end{equation}
Following conventional definitions in information theory, the Vendi score corresponds to the exponential of the \emph{Von Neumann entropy} of normalized kernel matrix $\frac{1}{n}K$. More generally, \cite{jalali_information-theoretic_2023} define the Rényi Kernel Entropy (RKE) score by applying order-2 Rényi entropy to this matrix, which reduces to the inverse-squared Frobenius norm of the normalized kernel matrix:
\begin{equation}\label{Eq: Intro-Vendi Score}
    \mathrm{RKE}(x_1,\ldots ,x_n) \, := \frac{1}{\Bigl\Vert \frac{1}{n}K \Bigr\Vert^2_F}
\end{equation}

Although the Vendi and RKE scores do not require reference samples, their computational cost increases rapidly with the number of generated samples $n$. Specifically, calculating the Vendi score for the  $n \times n$  kernel matrix $K$ generally involves an eigendecomposition of $K$, requiring  $O(n^3)$  computations. Therefore, the computational load of Vendi score becomes substantial for a large sample size $n$, and the Vendi score is typically evaluated for sample sizes limited to 20,000. In other words, the  Vendi score, as defined in Equation~\eqref{Eq: Intro-Vendi Score}, would be \emph{computationally infeasible} to compute with standard processors for sample sizes greater than a few tens of thousands. 

Following the above discussion, a key question that arises is whether the Vendi score estimated from restricted sample sizes (i.e. $n\le 20000$) has converged to its asymptotic value with infinite samples, which we call the \emph{population Vendi}. However, the statistical convergence of the Vendi score has not been thoroughly investigated in the literature. 
In this work, we study the statistical convergence of the Vendi and RKE diversity scores and aim to analyze the concentration of the estimated scores from a limited number of generated samples $n\lessapprox 20000$. 

\vspace{-2mm}
\subsection{Our Results on Vendi's Convergence}\vspace{-2mm}
We discuss the answer to the Vendi convergence question for two types of kernel functions: 1) kernel functions with a finite feature dimension, e.g. the cosine similarity and polynomial kernels, 2) kernel functions with an infinite feature map such as Gaussian (RBF) kernels. For kernel functions with a finite feature dimension $d$, we theoretically and numerically show that a sample size $n=O(d)$ is sufficient to guarantee convergence to the population Vendi (asymptotic value when $n\rightarrow\infty$). For example, the left plot in Figure~\ref{fig:kernel_convergence} shows that in the case of the cosine similarity kernel, the Vendi score on $n$ randomly selected ImageNet \citep{deng2009imagenet} samples has almost converged as the sample size reaches 5000, where the dimension $d$ (using standard DINOv2 embedding \citep{oquab_dinov2_2023}) is 768. 

In contrast, our numerical results for kernel functions with an infinite feature map demonstrate that for standard datasets, a sample size bounded by 20,000 could be insufficient for convergence of the Vendi score. For example, the right plot of Figure~\ref{fig:kernel_convergence} shows the evolution of the Vendi score with the Gaussian kernel on ImageNet data, and the score continues to grow at a significant rate with 20,000 samples\footnote{The heavy computational cost prohibits an empirical evaluation of the sample size required for Vendi’s convergence.}.

Observing the difference between Vendi score convergence  for finite and infinite-dimension kernel functions, a natural question is how to extend the definition of Vendi score from finite to infinite dimension case such that the diversity score would statistically converge in both scenarios. We attempt to address the question by introducing an alternative Vendi statistic, which we call the \emph{$t$-truncated Vendi score}.
The $t$-truncated Vendi score is defined using only the top-$t$ $\lambda_1\ge \cdots \ge \lambda_t$ eigenvalues of the kernel matrix, where $t$ is an integer hyperparameter. This modified score is defined as  
\begin{equation*}\label{Eq: Intro-TruncatedVendi Score}    \mathrm{Truncated}\text{-}\mathrm{Vendi}^{(t)}(x_1,\ldots ,x_n)  = \exp\Bigl(\sum_{i=1}^t {\lambda^{\scriptscriptstyle \text{trunc}}_i} \log\frac{1}{\lambda^{\scriptscriptstyle \text{trunc}}_i}\Bigr)
\end{equation*}
where we shift each of the top-$t$ eigenvalue $\lambda^{\scriptscriptstyle \text{trunc}}_i = \lambda_i +c $ by the same constant $c=\bigl(1-\sum_{i=1}^t\lambda_i\bigr)/t$ to ensure they add up to $1$ and provide a valid probability model. Observe that for a finite kernel dimension $d$ satisfying $d\le t$, the truncated and original Vendi scores take the same value, because the truncation will have no impact on the eigenvalues. On the other hand, under an infinite kernel dimension, the two scores may take different values.

As a main theoretical result, we prove that a sample size $n=O(t)$ is always enough to estimate the \emph{$t$-truncated population Vendi} from $n$ empirical samples, regardless of the finiteness of the kernel feature dimension. This result shows that the  \emph{$t$-truncated} Vendi score provides a statistically converging extension of the Vendi score from the finite kernel dimension to the infinite dimension case. To connect the defined $t$-truncated Vendi score to existing computation methods for the original Vendi score, we show that the existing computationally-efficient methods for computing the Vendi score can be viewed as approximations of our defined \emph{$t$-truncated} Vendi. Specifically, we show that the Nyström method in \citep{friedman_vendi_2023} and the FKEA method proposed by \cite{ospanov_fkea_2024} provide an estimate of the $t$-truncated Vendi. 

\begin{figure*}[t]
    \centering
    \includegraphics[width=\linewidth]{figures_final/vendi_t_diagram.pdf}
    \caption{Computation of the proposed $t$-truncated Vendi score. The kernel similarity matrix eigenspectrum is truncated, and the mass of the truncated tail (excluding the top-$t$ eigenvalues) is uniformly redistributed among the top-$t$ eigenvalues.}
    \label{fig:truncated vendi diagram}
\end{figure*}
\vspace{-2mm}
\subsection{Our Results on RKE's Convergence}
For the RKE score, we prove a universal convergence guarantee that holds for every kernel function. The theoretical guarantee shows that the RKE score, and more generally every order-$\alpha$ entropy score with $\alpha\ge 2$, will converge to its population value within $O(\frac{1}{\sqrt{n}})$ error for $n$ samples. Our theoretical guarantee also transfers to the truncated version of the RKE score. However, note that the truncation of the eigenspectrum becomes unnecessary in the RKE case, since the score enjoys universal convergence guarantees.  Figure~\ref{fig:kernel_convergence} shows that using both the cosine-similarity and Gaussian kernel functions, the RKE score nearly converges to its limit value with less than 10000 samples.  

Finally, we present the findings of several numerical experiments to validate our theoretical results on the convergence of Vendi, truncated Vendi, and RKE scores. Our numerical results on standard image, text, and video datasets and generative models indicate that in the case of a finite-dimension kernel map, the Vendi score can converge to its asymptotic limit, in which case, as we explained earlier, the Vendi score is identical to the truncated Vendi. On the other hand, in the case of infinite-dimension Gaussian kernel functions, we numerically observe the growth of the score beyond $n=$10,000. Our numerical results further confirm that the scores computed by Nyström method in \citep{friedman_vendi_2023} and the FKEA method \citep{ospanov_fkea_2024} provide tight estimations of the population truncated Vendi. The following summarizes this work's contributions:
\begin{itemize}[leftmargin=*]
    \item Analyzing the statistical convergence of Vendi and RKE diversity scores under restricted sample sizes $n\lessapprox 2\times 10^4$,
    \item Providing numerical evidence on the Vendi score's lack of convergence for infinite-dimensional kernel functions, e.g. the Gaussian (RBF) kernel,
    \item Introducing the truncated Vendi score as a statistically converging extension of the Vendi score from finite to infinite dimension kernel functions,
    \item Demonstrating the universal convergence of the RKE diversity score across all kernel functions.
\end{itemize}

\vspace{-3mm}
\section{Related Works}

\textbf{Diversity evaluation for generative models} Diversity evaluation in generative models can be categorized into two primary types: reference-based and reference-free methods. Reference-based approaches rely on a predefined dataset to assess the diversity of generated data. Metrics such as FID \citep{heusel_gans_2018}, KID and distributed KID \citep{binkowski2018demystifying,wang2023distributed} measure the distance between the generated data and the reference, while Recall \citep{sajjadi_assessing_2018, kynkaanniemi_improved_2019} and Coverage~\citep{naeem_reliable_2020} evaluate the extent to which the generative model captures existing modes in the reference dataset. \cite{pillutla2021mauve,pillutla-etal:mauve:jmlr2023} propose MAUVE metric that uses information divergences in a quantized embedding space to measure the gap between generated data and reference distribution. In contrast, the reference-free metrics, Vendi \citep{friedman_vendi_2023} and RKE \citep{jalali_information-theoretic_2023}, assign diversity scores based on the eigenvalues of a kernel similarity matrix of the generated data. \cite{jalali_information-theoretic_2023} interpret the approach as identifying modes and their frequencies within the generated data followed by entropy calculation for the frequency parameters. The Vendi and RKE scores have been further extended to quantify the diversity of conditional prompt-based generative AI models \citep{ospanov2024dissecting,jalali2024conditional} and to select generative models in online settings \citep{rezaei2024more,hu2024online,hu2025multi}. Also, \citep{zhang2024interpretable,zhang2025unveiling,jalali2025towards,gong2025kernel,wu2025fusingcrossmodalunimodalrepresentations} extend the entropic kernel-based scores to measure novelty and embedding dissimilarity. In our work, we specifically focus on the statistical convergence of the vanilla Vendi and RKE scores.

\textbf{Statistical convergence analysis of kernel matrices' eigenvalues.} The convergence analysis of the eigenvalues of kernel matrices has been studied by several related works. \cite{shawe2005eigenspectrum} provide a concentration bound for the eigenvalues of a kernel matrix. We note that the bounds in \citep{shawe2005eigenspectrum} use the expectation of eigenvalues $\mathbb{E}_m[\hat{\boldsymbol{\lambda}}(S)]$ for a random dataset $S=(\mathbf{x}_1,\dots,\mathbf{x}_m)$ of fixed size $m$ as the center vector in the concentration analysis. However, since eigenvalues are non-linear functions of a matrix, this concentration center vector $\mathbb{E}_m[\hat{\boldsymbol{\lambda}}(S)]$ does not match the eigenvalues of the asymptotic kernel matrix as the sample size approaches to infinity. On the other hand, our convergence analysis focuses on the asymptotic eigenvalues with an infinite sample size, which determines the limit value of Vendi scores.  
In another related work, \cite{bach_information_2022} discusses a convergence result for the Von-Neumann entropy of kernel matrix. While this result proves a non-asymptotic guarantee on the convergence of the entropy function, the bound may not guarantee convergence at standard sample sizes for computing Vendi scores (less than $10000$ in practice). In our work, we aim to provide convergence guarantees for the finite-dimension and generally truncated Vendi scores with restricted sample sizes.   

\textbf{Efficient computation of matrix-based entropy.} Several strategies have been proposed in the literature to reduce the computational complexity of matrix-based entropy calculations, which involve the computation of matrix eigenvalues—a process that scales cubically with the size of the dataset. \cite{dong2023optimal} propose an efficient algorithm for approximating matrix-based Renyi’s entropy of arbitrary order $\alpha$, which achieves a reduction in computational complexity down to $O(n^2sm)$ with $s,m\ll n$. Additionally, kernel matrices can be approximated using low-rank techniques such as incomplete Cholesky decomposition \citep{fine2001efficient, bach2002kernel} or CUR matrix decompositions \citep{curmatrix_michael}, which provide substantial computational savings. \cite{pasarkar2023cousins} suggest to leverage Nyström method \citep{williams2000nystrom} with $m$ components, which results in $O(nm^2)$ computational complexity. Further reduction in complexity is possible using Random Fourier Features, as suggested by \cite{ospanov_fkea_2024}, which allows the computation to scale linearly with $O(n)$ as a function of the dataset size. This work focuses on the latter two methods and the population quantities estimated by them.

\textbf{Impact of embedding spaces on diversity evaluation.} In our image-related experiments, we used the DinoV2 embedding \citep{oquab_dinov2_2023}, as \cite{stein_exposing_2023} demonstrate the alignment of this embedding with human evaluations. We note that the kernel function in the Vendi score can be similarly applied to other embeddings, including the standard InceptionV3\citep{szegedy_rethinking_2016} and CLIP embeddings \citep{radford_learning_2021} as suggested by \cite{kynkaanniemi_role_2022}. 

\vspace{-3mm}
\section{Preliminaries}

Consider a generative model $\mathcal{G}$ that generates samples from a probability distribution $P_X$. To conduct a reference-free evaluation of the model, we suppose the evaluator has access to $n$ independently generated samples from $P_X$, denoted by $x_1,\ldots ,x_n\in\mathcal{X}$. 
The assessment task is to estimate the diversity of generative model $\mathcal{G}$ by measuring the variety of the observed generated data, $x_1,\ldots x_n$. In the following subsections, we will discuss kernel functions and their application to define the Vendi and RKE diversity scores.

\subsection{Kernel Functions and Matrices}
Following the standard definition, $k:\mathcal{X}\times \mathcal{X} \rightarrow \mathbb{R}$ is called a kernel function if for every integer $n\in\mathbb{N}$ and inputs $x_1, \ldots, x_n \in \mathcal{X}$, the following kernel similarity matrix $K\in\mathbb{R}^{n\times n}$ 
is positive semi-definite (PSD):
\begin{equation}
    K = \begin{bmatrix} k(x_1,x_1) & \cdots & k(x_1,x_n) \\ \vdots & \ddots & \vdots \\ k(x_n,x_1) & \cdots & k(x_n,x_n)
    \end{bmatrix}
\end{equation}
Aronszajn’s Theorem \citep{aronszajn1950reproducing} shows that this definition is equivalent to the existence of a feature map $\phi :\mathcal{X} \rightarrow \mathbb{R}^d$ 
such that for every $x, x' \in\mathcal{X}$ we have the following where $\langle \cdot ,\cdot \rangle $ denotes the standard inner product in the $\mathbb{R}^d$ space:
\begin{equation}\label{Eq: Kernel Equivalent Definition}
    k(x,x') \, =\,  \bigl\langle \phi(x), \phi(x') \bigr\rangle 
\end{equation}
In this work, we study the evaluation using two types of kernel functions: 1) finite-dimension kernels where dimension $d$ is finite, 2) infinite-dimension kernels where there is no feature map satisfying \eqref{Eq: Kernel Equivalent Definition} with a finite $d$ value. A standard example of a finite-dimension kernel is the cosine similarity function where $\phi_{\text{cosine}}(x)= x/ \Vert x\Vert_2$. Also, a widely-used infinite-dimension kernel is the Gaussian (RBF) kernel with bandwidth parameter $\sigma >0$ defined as
\begin{equation}
    k_{\text{\rm Gaussian}(\sigma)} \bigl(x , x'\bigr) \, :=\, \exp\Bigl(-\frac{\bigl\Vert x - x'\bigr\Vert^2_2}{2\sigma^2}\Bigr)
\end{equation}
Both the mentioned kernel examples belong to normalized kernels which require $k(x, x)=1$ for every $x$, i.e., the feature map $\phi(x)$ has unit Euclidean norm for every $x$. Given a normalized kernel function, the non-negative eigenvalues of the normalized kernel matrix $\frac{1}{n}K$ for $n$ points $x_1,\ldots x_n$ will sum up to $1$, i.e., they form a probability model.

\subsection{Matrix-based Entropy Functions and Vendi Score}
For a PSD matrix $A \in\mathbb{R}^{d\times d}$ with unit trace $\mathrm{Tr}(A)=1$, $A$'s eigenvalues form a probability model. The order-$\alpha$ Renyi entropy of matrix $A$ is defined using the order-$\alpha$ entropy of its eigenvalues as 
\begin{equation}\label{Eq: order-alpha entropy}
    H_{\alpha}(A) \, :=\, \frac{1}{1-\alpha}\log\Bigl(\sum_{i=1}^d \lambda^\alpha_i\Bigr)
\end{equation}
For the special case $\alpha=2$, one can consider the Frobenius norm $\Vert\cdot \Vert_F$ and apply the identity $\bigl\Vert A\bigr\Vert_F^2 = \sum_{i=1}^d \lambda^2_i$ to show $H_{2}(A)=\log\bigl(1/\bigl\Vert A\bigr\Vert_F^2\bigr)$. Moreover, for $\alpha =1$, the above definition reduces to the Shannon entropy of the eigenvalues as
$    H_{1}(A) \, := \, \sum_{i=1}^d \lambda_i \log ({1}/{\lambda_i})$ \citep{renyi1961measures}. 

\cite{jalali_information-theoretic_2023} applies the above definition for order $\alpha=2$ to the normalized kernel similarity matrix $\frac{1}{n}K$ to define the RKE diversity score (called RKE mode count), which reduces to
\begin{equation}
   \mathrm{RKE}(x_1,\ldots , x_n) \, :=\, \exp\Bigl(H_2\bigl(\frac{1}{n}K\bigr)\Bigr) =\,  \Bigl\Vert \frac{1}{n}K\Bigr\Vert^{-2}_F 
\end{equation}

For a general entropy order $\alpha$,
\citep{friedman_vendi_2023,pasarkar2023cousins} apply the matrix-based entropy definition to the normalized kernel matrix $\frac{1}{n}K$ and define the order-$\alpha$ Vendi score for samples $x_1,\ldots , x_n$ as
\begin{equation}
    \mathrm{Vendi}_\alpha \bigl(x_1,\ldots ,x_n\bigr) \, :=\, \exp\Bigl(H_\alpha\bigl(\frac{1}{n} K \bigr)\Bigr)
\end{equation}
Specifically, for order $\alpha=1$, the above definition results in the standard (order-1) Vendi score in Equation~\eqref{Eq: Intro-Vendi Score}.

\subsection{Statistical Analysis of Vendi Score}
To derive the population limits of Vendi and RKE scores under infinite sampling, which we call \emph{population Vendi} and \emph{population RKE}, respectively, we review the following discussion from \citep{bach_information_2022,jalali_information-theoretic_2023}. First, note that the normalized kernel matrix $\frac{1}{n}K$, whose eigenvalues are used in the definition of Vendi score, can be written as:
\begin{equation}
    \frac{1}{n}K = \frac{1}{n}\Phi \Phi^\top
\end{equation}
where $\Phi\in\mathbb{R}^{n\times d}$ is an $n\times d$ matrix whose rows are the feature presentations of samples, i.e., $\phi(x_1),\ldots ,\phi(x_n)$. Therefore, the normalized kernel matrix $\frac{1}{n}K$ shares the same non-zero eigenvalues with $\frac{1}{n} \Phi^\top \Phi$, where the multiplication order is flipped. Note that $\frac{1}{n} \Phi^\top \Phi$ is equal to the empirical kernel covariance matrix $ \widehat{C}_X$ defined as:
\begin{equation*}
    \widehat{C}_X := \frac{1}{n}\sum_{i=1}^n \phi(x_i)\phi(x_i)^\top  = \frac{1}{n} \Phi^\top \Phi.
\end{equation*}
As a result, the empirical covariance matrix $\widehat{C}_X=\frac{1}{n} \Phi^\top \Phi$ and kernel matrix $\frac{1}{n}K=\frac{1}{n} \Phi \Phi^\top$ share the same non-zero eigenvalues and therefore have the same matrix-based entropy value for every order $\alpha$: $
    H_{\alpha}(\frac{1}{n}K) = H_{\alpha}(\widehat{C}_X)
$. 
Therefore, if we consider the population kernel covariance matrix $\widetilde{C}_X = \mathbb{E}_{x\sim P_X}\bigl[\phi(x)\phi(x)^\top\bigr]$, we can define the population Vendi score as follows.
\begin{definition}
    Given data distribution $P_X$, we define the order-$\alpha$ population Vendi, ${\mathrm{Vendi}}_\alpha(P_X)$, using the matrix-based entropy of the population kernel covariance matrix $\widetilde{C}_X = \mathbb{E}_{x\sim P_X}\bigl[\phi(x)\phi(x)^\top\bigr]$ as
    \begin{equation}
        {\mathrm{Vendi}}_\alpha(P_X) \, :=\, \exp\Bigl( H_{\alpha}(\widetilde{C}_X)\Bigr)
    \end{equation}
\end{definition}
Note that the population RKE score is identical to the population $\mathrm{Vendi}_2$, since RKE and $\mathrm{Vendi}_2$ are the same. 

\vspace{-3mm}
\section{Statistical Convergence of Vendi and RKE Scores}
Given the definitions of the Vendi score and the population Vendi, a relevant question is how many samples are required to accurately estimate the population Vendi using the Vendi score. To address this question, we first prove the following concentration bound on the vector of ordered eigenvalues $[\lambda_1,\ldots,\lambda_n]$ of the kernel matrix for a normalized kernel function. 

\begin{theorem}\label{Thm: eigenvalue convergence}
Consider a normalized kernel function $k$ satisfying $k(x,x)=1$ for every $x\in\mathcal{X}$. Let $\widehat{\boldsymbol{\lambda}}_n$ be the vector of sorted eigenvalues of the normalized kernel matrix $\frac{1}{n}K$ for $n$ independent samples $x_1,\ldots ,x_n\sim P_X$. If we define $\widetilde{\boldsymbol{\lambda}}$ as the vector of sorted eigenvalues of underlying covariance matrix $\widetilde{C}_X$, then if $n\ge 2+8\log (1/\delta)$, the following inequality holds with probability at least $1-\delta$:
\begin{equation*}
    \bigl\Vert \widehat{\boldsymbol{\lambda}}_n - \widetilde{\boldsymbol{\lambda}} \bigr\Vert_2 \, \le \, \sqrt{\frac{32\log\bigl(2/\delta\bigr)}{n}}
\end{equation*}
Note that in calculating the subtraction $\widehat{\boldsymbol{\lambda}}_n - \widetilde{\boldsymbol{\lambda}}$, we add $|d-n|$ zero entries to the lower-dimension vector, if the dimension of vectors $\widehat{\boldsymbol{\lambda}}_n$ and $\widetilde{\boldsymbol{\lambda}}$ do not match.
\end{theorem}
\begin{proof}
    We defer the proof to the Appendix.
\end{proof}

Theorem~\ref{Thm: eigenvalue convergence} results in the following corollary on a \emph{dimension-free convergence guarantee} for every $\mathrm{Vendi}_\alpha$ score with order $\alpha\ge 2$, including the RKE score (i.e. $\mathrm{Vendi}_2$).

\begin{corollary}\label{Corollary: Order greater than 2}
In the setting of Theorem \ref{Thm: eigenvalue convergence}, for every $\alpha\ge 2$ and $n\ge 2+8\log (1/\delta)$, the following bound holds with probability at least $1-\delta$:\vspace{-1mm}
\begin{align*}
    \Bigl\vert \mathrm{Vendi}_\alpha\bigl(x_1,\ldots,x_n\bigr)^{\frac{1-\alpha}{\alpha}} - \mathrm{Vendi}_\alpha\bigl(P_X\bigr)^{\frac{1-\alpha}{\alpha}} \Bigr\vert
    \le  &\sqrt{\frac{32\log\frac{2}{\delta}}{n}}
\end{align*}
Notably, for $\alpha=2$, we arrive at the following bound on the gap between the empirical and population RKE scores:
\begin{align*}
    \Bigl\vert \mathrm{RKE}\bigl(x_1,\ldots,x_n\bigr)^{-1/2} - \mathrm{RKE}\bigl(P_X\bigr)^{-1/2} \Bigr\vert \le  \sqrt{\frac{32\log\frac{2}{\delta}}{n}}
\end{align*}
\end{corollary}
\begin{proof}
    We defer the proof to the Appendix.
\end{proof}
Therefore, the bound in Corollary \ref{Corollary: Order greater than 2} holds regardless of the dimension of kernel feature map, indicating that the RKE score enjoys a universal convergence guarantee across all kernel functions. Next, we show that Theorem~\ref{Thm: eigenvalue convergence} implies the following corollary on a dimension-dependent convergence guarantee for order-$\alpha$ Vendi score with $1\le \alpha <  2$, including standard (order-$1$) Vendi score.
\begin{corollary}\label{Corollary: Finite Dimension}
In the setting of Theorem \ref{Thm: eigenvalue convergence}, consider a finite dimension kernel map where we suppose $\mathrm{dim}(\phi)=d<\infty$. (a) For $\alpha=1$, assuming $n\geq 32e^2\log(2/\delta)$, the following bound holds with probability at least $1-\delta$:
\begin{align*}
    &\Bigl\vert\, \log\bigl(\mathrm{Vendi}_1\bigl(x_1,\ldots,x_n\bigr)\bigr) - \log\bigl(\mathrm{Vendi}_1\bigl(P_X\bigr)\bigr)\, \Bigr\vert \\
    \le\:  &\sqrt{\frac{8d\log\bigl(2/\delta\bigr)}{n}}\log\Bigl(\frac{nd}{32\log(2/\delta)}\Bigr).
\end{align*}
(b) For every $1< \alpha< 2$ and $n\ge 2+8\log (1/\delta)$, the following bound holds with probability at least $1-\delta$:
\begin{align*}
    &\Bigl\vert\, \mathrm{Vendi}_\alpha\bigl(x_1,\ldots,x_n\bigr)^{\frac{1-\alpha}{\alpha}} - \mathrm{Vendi}_\alpha\bigl(P_X\bigr)^{\frac{1-\alpha}{\alpha}}\, \Bigr\vert \\
    \le\:  &\sqrt{\frac{32d^{2-\alpha}\log\bigl(2/\delta\bigr)}{n}}
\end{align*}
\end{corollary}
\begin{proof}
    We defer the proof to the Appendix.
\end{proof}

Therefore, assuming a finite feature map $d<\infty$ and given an entropy order $1\le \alpha<2$, the above results indicate the convergence of the Vendi score to the underlying population Vendi given $n=O(d^{2-\alpha})$ samples. Observe that this result is consistent with our numerical observations of the convergence of Vendi score using the finite-dimension cosine similarity kernel in Figure~\ref{fig:kernel_convergence}. 

\vspace{-3mm}
\section{Truncated Vendi Score and its Estimation via Proxy Kernels}

Corollaries \ref{Corollary: Order greater than 2}, \ref{Corollary: Finite Dimension} demonstrate that if the Vendi score order $\alpha$ is greater than $2$ or the kernel feature map dimension $d$ is finite, then the Vendi score can converge to the population Vendi with $n=O(d)$ samples. However, the theoretical results do not apply to an order $1\le \alpha < 2$ when the kernel map dimension is infinite, e.g. the original order-1 Vendi score \citep{friedman_vendi_2023} with a Gaussian kernel. Our numerical observations indicate that a standard sample size below 20000 could be insufficient for the convergence of order-1 Vendi score (Figure~\ref{fig:kernel_convergence}). To address this gap, here we define the truncated Vendi score by truncating the eigenspectrum of the kernel matrix, and then show that  the existing kernel approximation algorithms for Vendi score concentrate around this modified Vendi score. 
\begin{definition}\label{Definition: population Vendi}
Consider the normalized kernel matrix $\frac{1}{n}K$ of samples $x_1,\ldots ,x_n$. Then, for an integer parameter $t \ge 1$, consider the top-$t$ eigenvalues of $\frac{1}{n}K$: $\lambda_1\ge \lambda_2\ge \cdots \ge \lambda_t$. Define $S_t = \sum_{i=1}^t \lambda_i$ and consider the truncated probability sequence $[{\lambda}^{\scriptscriptstyle \text{trunc}}_1,\ldots ,{\lambda}^{\scriptscriptstyle \text{trunc}}_t]$: $${\lambda}^{\scriptscriptstyle \text{trunc}}_i= \lambda_i + \frac{1-S_t}{t}\quad \text{for $\; i=1,\ldots ,t$}$$
We define the order-$\alpha$  $t$-truncated Vendi score as
\begin{equation*}
    \mathrm{Vendi}_\alpha^{(t)}(x_1,\ldots ,x_n) := \exp\Bigl(\frac{1}{1-\alpha}\log\Bigl(\sum_{i=1}^t {\lambda}^{{\scriptscriptstyle \text{trunc}}^{\large\alpha}}_i\Bigr)\Bigr)
\end{equation*}
Notably, for order $\alpha=1$, the $t$-truncated Vendi score is:
\begin{equation*}
    \mathrm{Vendi}_1^{(t)}(x_1,\ldots ,x_n) := \exp\Bigl(\sum_{i=1}^t {\lambda}^{\scriptscriptstyle \text{trunc}}_i\log\frac{1}{{\lambda}^{\scriptscriptstyle \text{trunc}}_i}\Bigr)
\end{equation*}
\end{definition}

\begin{remark}\label{Remark: truncated Vendi statistic}
The above definition of $t$-truncated Vendi score leads to the definition of $t$-truncated population Vendi $\mathrm{Vendi}_\alpha^{(t)}(P_X)$, where the mentioned truncation process is applied to the eigenspectrum of the population kernel covariance matrix $\widetilde{C}_X$.
 Note that the truncated Vendi score is a statistic and function of random samples $x_1,\ldots ,x_n$, whereas the truncated population Vendi is deterministic and a characteristic of the population distribution $P_X$.
\end{remark}
According to Definition~\ref{Definition: population Vendi}, we find the probability model with the minimum $\ell_2$-norm difference from the $t$-dimensional vector $[\lambda_1,\ldots ,\lambda_t]$ including only the top-$t$ eigenvalues. Then, we use the order-$\alpha$ entropy of the probability model to define the order-$\alpha$ $t$-truncated population Vendi. Our next result shows that this population quantity can be estimated using $n=O(t)$ samples by $t$-truncated Vendi score for every kernel function.
\begin{theorem}\label{Thm: truncated Vendi guarantee}
Consider the setting in Theorem \ref{Thm: eigenvalue convergence}. Then, for every $n\ge 2+8\log (1/\delta)$, the difference between the  $t$-truncated population Vendi and the empirical $t$-truncated Vendi score of samples $x_1,\ldots ,x_n$ is bounded with probability at least $1-\delta$:
\begin{align*}
    &\Bigl\vert\, \mathrm{Vendi}^{(t)}_\alpha\bigl(x_1,\ldots,x_n\bigr)^{\frac{1-\alpha}{\alpha}} - \mathrm{Vendi}^{(t)}_\alpha\bigl(P_X\bigr)^{\frac{1-\alpha}{\alpha}}\, \Bigr\vert \\
    \le\:  &\sqrt{\frac{32\max\{1,t^{2-\alpha}\}\log\bigl(2/\delta\bigr)}{n}}
\end{align*}
\end{theorem}
\begin{proof}
    We defer the proof to the Appendix.
\end{proof}
As implied by Theorem~\ref{Thm: truncated Vendi guarantee}, the $t$-truncated population Vendi can be estimated using $O(t)$ samples, i.e. the truncation parameter $t$ plays the role of the bounded dimension of a finite-dimension kernel map. Our next theorem shows that the Nyström method \citep{friedman_vendi_2023} and the FKEA method \citep{ospanov_fkea_2024} for reducing the computational costs of Vendi scores have a bounded difference with the truncated population Vendi. 

\begin{theorem}\label{Thm: Nyström, FKEA}
    Consider the setting of Theorem~\ref{Thm: eigenvalue convergence}. (a) Assume that the kernel function is shift-invariant and the FKEA method with $t$ random Fourier features is used to approximate the Vendi score. Then, for every $\delta$ satisfying $n\ge 2+8\log (1/\delta)$, with probability at least $1-\delta$: 
    \begin{align*}
    &\Bigl\vert \mathrm{FKEA}\text{-}\mathrm{Vendi}^{(t)}_\alpha\bigl(x_1,\ldots,x_n\bigr)^{\frac{1-\alpha}{\alpha}}- \mathrm{Vendi}^{(t)}_\alpha\bigl(P_X\bigr)^{\frac{1-\alpha}{\alpha}} \Bigr\vert \\
    &\, \le \sqrt{\frac{128\max\{1,t^{2-\alpha}\}\log\bigl(3/\delta\bigr)}{\min\{n,t\}}}
\end{align*}
(b) Assume that the Nyström method is applied with parameter $t$ for approximating the kernel function. Then, if for some $r\ge 1$, the kernel matrix $K$'s $r$th-largest  eigenvalue satisfies ${\lambda}_{r}  \le \tau $ and $t\ge r\tau \log(n)$, the following holds with probability at least $1-\delta - 2n^{-3}$:  
\begin{align*}
    &\Bigl\vert \mathrm{ Nystrom}\text{-}\mathrm{Vendi}^{(t)}_\alpha\bigl(x_1,\ldots,x_n\bigr)^{\frac{1-\alpha}{\alpha}} - \mathrm{Vendi}^{(t)}_\alpha\bigl(P_X\bigr)^{\frac{1-\alpha}{\alpha}}\Bigr\vert  \\
    & \ \le \mathcal{O}\Bigl(\sqrt{\frac{\max\{1,t^{2-\alpha}\}\log\bigl(2/\delta\bigr)t\tau^2\log(n)^2}{n}}\Bigr)
\end{align*}
\end{theorem}
\begin{proof}
    We defer the proof to the Appendix.
\end{proof}

\vspace{-3mm}
\section{Numerical Results}
\begin{figure*}
    \includegraphics[width=0.95\linewidth]{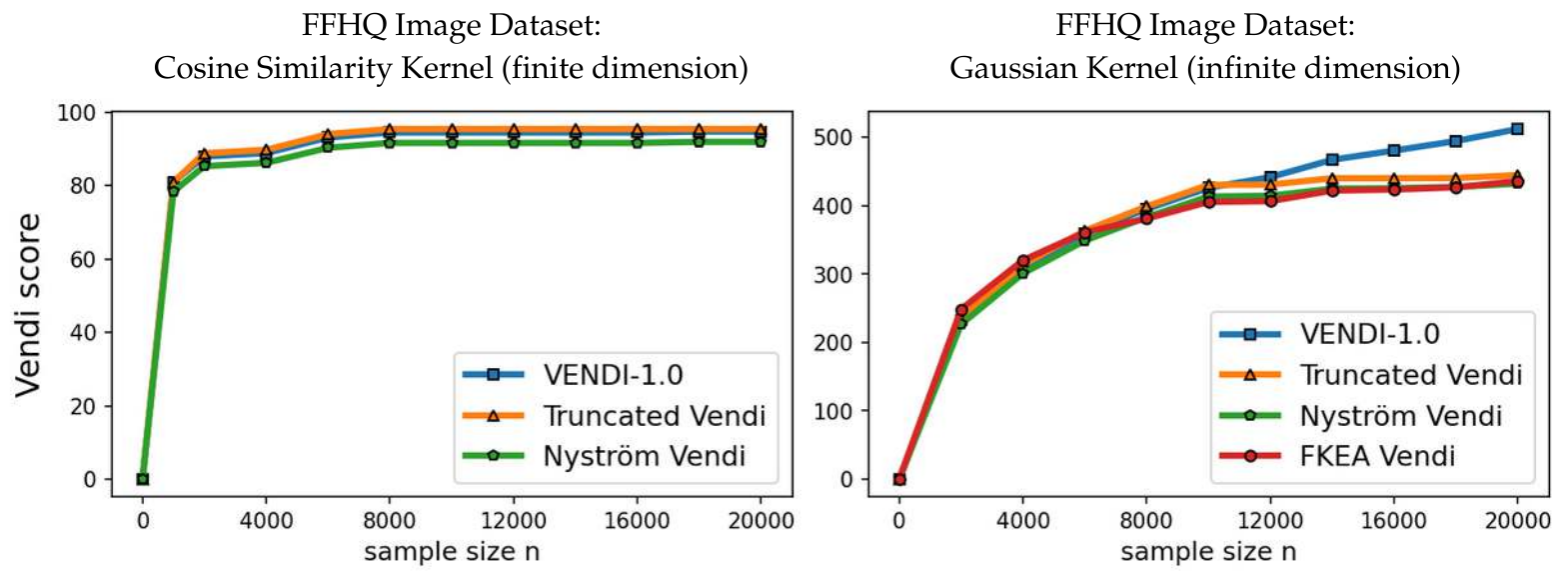}
    \caption{Statistical convergence of Vendi score for different sample sizes on FFHQ\citep{karras2019style} data: (Left plot) finite-dimension cosine similarity kernel (Right plot) infinite dimension Gaussian kernel with bandwidth $\sigma=35$. \emph{DINOv2} embedding (dimension 768) is used in computing the scores.}
    \label{VENDI_ffhq_convergence}
\end{figure*}

\begin{figure*}
    \includegraphics[width=0.95\linewidth]{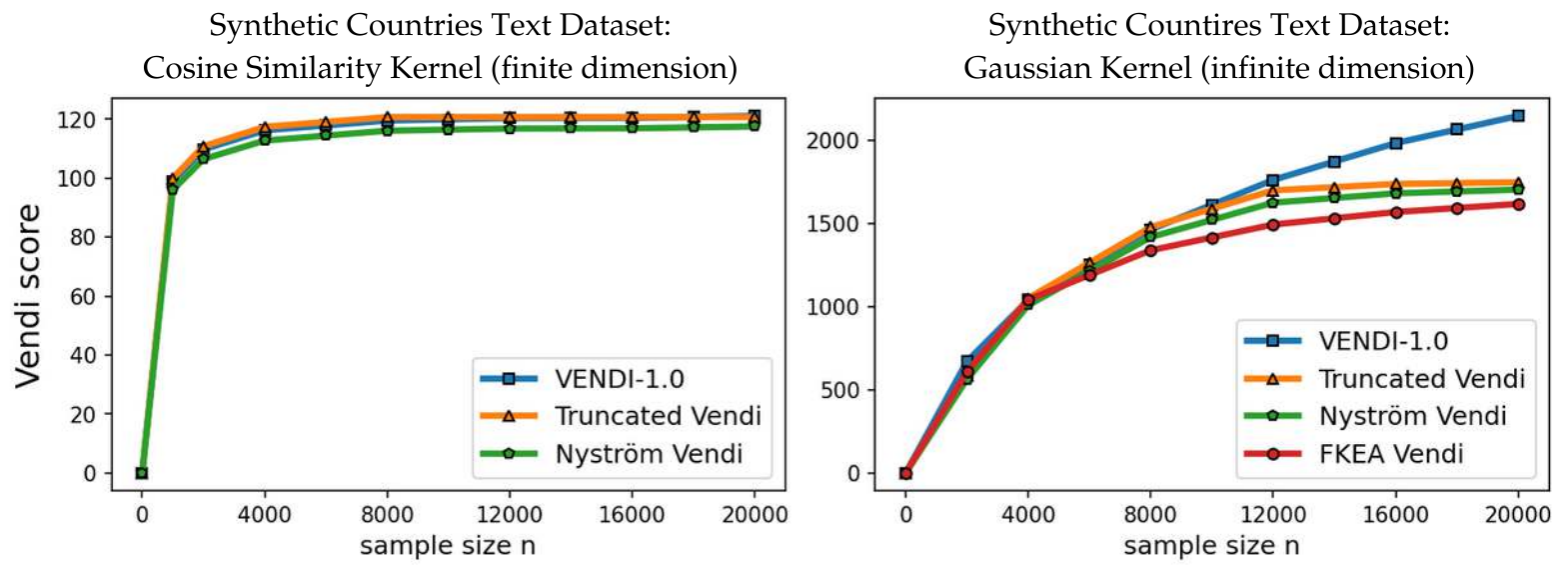}
    \caption{Statistical convergence of Vendi score for different sample sizes on Synthetic Countries data: (Left plot) finite-dimension cosine similarity kernel (Right plot) infinite dimension Gaussian kernel with bandwidth $\sigma=0.6$. \emph{text-embedding-3-large} embedding (dimension 3072) is used in computing the scores.}
    \label{VENDI_countries_convergence}
\end{figure*}

\begin{figure*}
    \includegraphics[width=0.95\linewidth]{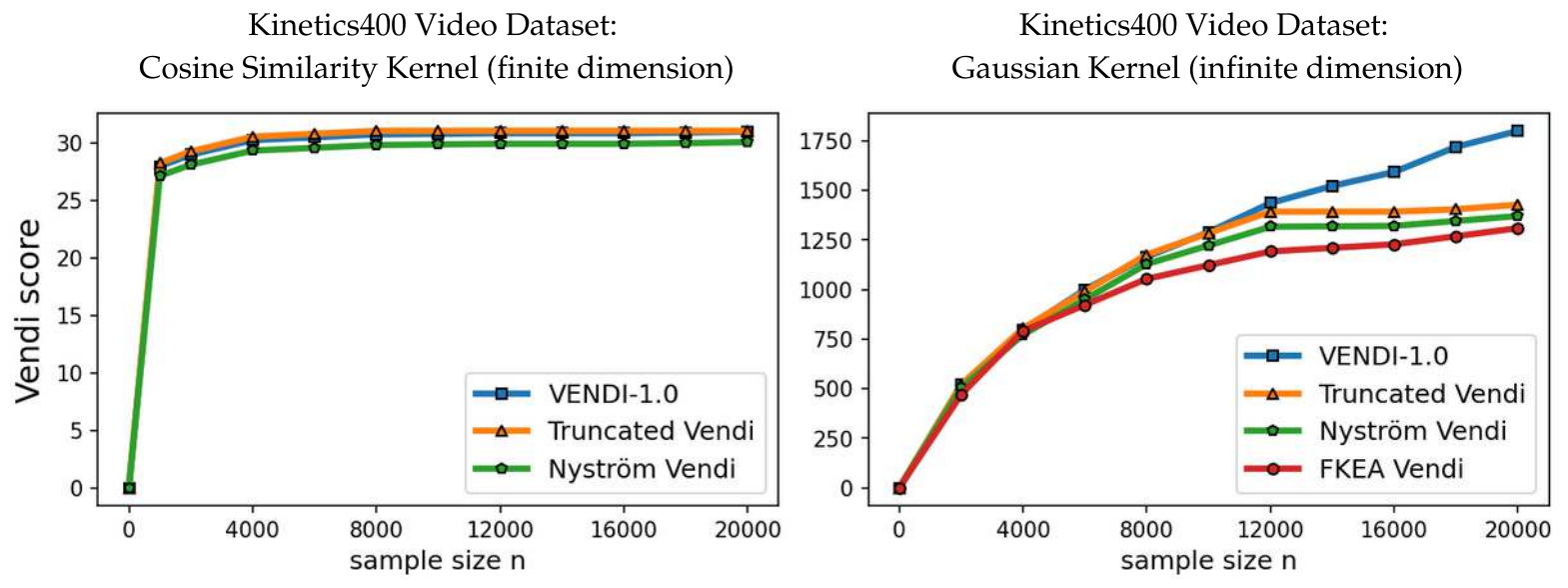}
    \caption{Statistical convergence of Vendi score for different sample sizes on Kinetics400\citep{kay2017kinetics} data: (Left plot) finite-dimension cosine similarity kernel (Right plot) infinite dimension Gaussian kernel with bandwidth $\sigma=4.0$. \emph{I3D} embedding (dimension 1024) is used in computing the scores.}
    \label{VENDI_k400_convergence}
\end{figure*}

\begin{figure*}
    \centering
    \includegraphics[width=0.83\linewidth]{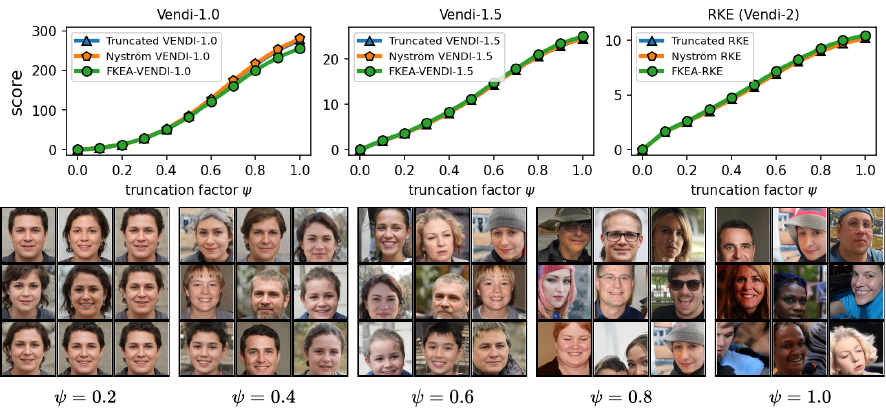}
    \caption{Diversity evaluation of Vendi scores on truncated StyleGAN3 generated FFHQ dataset with varying truncation coefficient $\psi$. Fixed sample size $n=$20k is used for estimating the scores.
    }  \vspace{-4mm}\label{VENDI_ffhq_truncation}
\end{figure*}

\begin{figure*}
    \centering
    \includegraphics[width=0.83\linewidth]{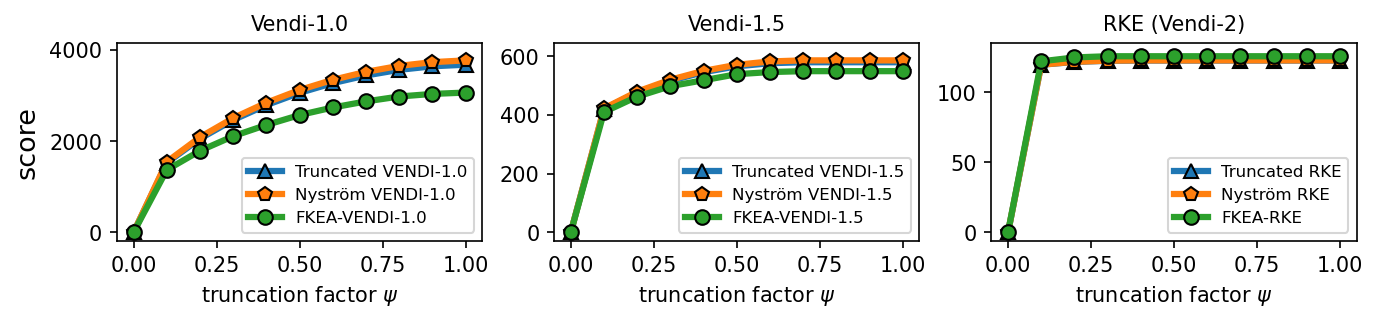}
    \caption{Diversity evaluation of Vendi scores on truncated StyleGAN-XL generated ImageNet dataset with varying truncation coefficient $\psi$. Fixed sample size $n=$20k is used for estimating the scores.
    }
    \vspace{-5mm}
    \label{VENDI_imagenet_truncation}
\end{figure*}

We evaluated the convergence of the Vendi score, the truncated Vendi score, and the proxy Vendi scores using the Nyström method and FKEA in our numerical experiments. We provide a comparative analysis of these scores across different data types and models, including image, text, and video. In our experiments, we considered the cosine similarity kernel as a standard kernel function with a finite-dimension map and the Gaussian (RBF) kernel as a kernel function with an infinite-dimension feature map. In the experiments with Gaussian kernels, we matched the kernel bandwidth parameter with those chosen by \citep{jalali_information-theoretic_2023,ospanov_fkea_2024} for the same datasets. We used 20,000 number of samples per score computation, consistent with standard practice in the literature. 
To investigate how computation-cutting methods compare to each other, in the experiments we matched the truncation parameter $t$ of our defined $t$-truncated Vendi score with the Nyström method's hyperparameter on the number of randomly selected rows of kernel matrix and the FKEA's hyperparameter of the number of random Fourier features. The Vendi and FKEA implementations were adopted from the corresponding references' GitHub webpages, while the Nyström method was adopted from the \texttt{scikit-learn} Python package.


\subsection{Convergence Analysis of Vendi Scores}

To assess the convergence of the discussed Vendi scores, we conducted experiments on four datasets including ImageNet and FFHQ~\citep{karras2019style} image datasets, a synthetic text dataset with 400k paragraphs generated by GPT-4 about 100 randomly selected countries, and the Kinetics video dataset \citep{kay2017kinetics}. Our results, presented in Figures \ref{VENDI_ffhq_convergence}, \ref{VENDI_countries_convergence}, and \ref{VENDI_k400_convergence}, show that for the finite-dimension cosine similarity kernel the Vendi score converges rapidly to the underlying value and the proxy versions including truncated and Nyström Vendi scores were almost identical to the original Vendi score. This observation is consistent with our theoretical results on the convergence of Vendi scores under finite-dimension kernel maps. On the other hand, in the case of infinite dimension Gaussian kernel, we observed that the $\mathrm{Vendi}_1$ score did not converge using 20k samples and the score value kept growing with a considerable rate. However, the $t$-truncated Vendi score with $t=10000$ converged to its underlying statistic shortly after 10000 samples were used. Consistent with our theoretical result, the proxy Nyström and FKEA estimated scores with their rank hyperparameter matched with $t$ also converged to the limit of the truncated Vendi scores. The numerical results show the connection between the truncated Vendi score and the existing kernel methods for approximating the Vendi score.

\subsection{Correlation between the truncated Vendi score and diversity of data}

We performed experiments to test the correlation between the truncated Vendi score and the ground-truth diversity of data. To do this, we applied the truncation technique to the FFHQ-based StyleGAN3~\citep{karras2021aliasfree} model and the ImageNet-based StyleGAN-XL~\citep{Sauer2021ARXIV} model and simulated generative models with different underlying diversity by varying the truncation technique. Considering the Gaussian kernel, we estimated the $t$-truncated Vendi score with $t=10000$ by averaging the estimated $t$-truncated Vendi scores over $5$ independent datasets of size 20k where the score seemed to converge to its underlying value. Figures~\ref{VENDI_ffhq_truncation},~\ref{VENDI_imagenet_truncation} show how the estimated statistic correlates with the truncation parameter for order-$\alpha$ Vendi scores with $\alpha = 1, \, 1.5, 2$. In all these experiments, the estimated truncated Vendi score correlated with the underlying diversity of the models. In addition, we plot the proxy Nyström and FKEA proxy Vendi values computed using 20000 samples which remain close to the estimated $t$-truncated statistic. These empirical results suggest that the estimated $t$-truncated Vendi score with Gaussian kernel can be used to evaluate the diversity of generated data. Also, the Nyström and FKEA methods were both computationally efficient in estimating the truncated Vendi score from limited generated data. We defer the presentation of the additional numerical results on the convergence of Vendi scores with different orders, kernel functions and embedding spaces to the Appendix.

\section{Conclusion}

In this work, we investigated the statistical convergence behavior of Vendi diversity scores estimated from empirical samples. We highlighted that, due to the high computational complexity of the score for datasets larger than a few tens of thousands of generated data points, the score is often calculated using sample sizes below 10,000. We demonstrated that such restricted sample sizes do not pose a problem for statistical convergence as long as the kernel feature dimension is bounded. However, our numerical results showed a lack of convergence to the population Vendi when using an infinite-dimensional kernel map, such as the Gaussian kernel. To address this gap, we introduced the truncated population Vendi as an alternative target quantity for diversity evaluation. We showed that existing Nyström and FKEA methods for approximating Vendi scores concentrate around this truncated population Vendi. An interesting future direction is to explore the relationship between other kernel approximation techniques and the truncated population Vendi. Also, a comprehensive analysis of the computational-statistical trade-offs involved in estimating the Vendi score is another relevant future direction.

\begin{acknowledgements} 
This work is partially supported by a grant from the Research Grants Council of the Hong Kong Special Administrative Region, China, Project 14209920, and is partially supported by CUHK Direct Research Grants with CUHK Project No. 4055164 and 4937054.
Finally, the authors thank the anonymous reviewers for their thoughtful feedback and constructive suggestions.

\end{acknowledgements}

\clearpage
\clearpage

\bibliography{references}

\newpage

\onecolumn
%
%




%

%

\title{Supplementary Material}
\maketitle
\begin{appendices}

\section{Proofs}
\subsection{Proof of Theorem \ref{Thm: eigenvalue convergence}}
To prove the theorem, we will use the following lemma followed from \citep{gross2011recovering,kohler2017sub}.
\begin{lemma}[Vector Bernstein Inequality \citep{gross2011recovering,kohler2017sub}]\label{Lemma: Vector Bernstein} Suppose that $z_1,\ldots ,z_n$ are independent and identically distributed random vectors with zero mean $\mathbb{E}[z_i] =\mathbf{0}$ and bounded $\ell_2$-norm $\Vert z_i\Vert_2\le c$. Then, for every $0\le\epsilon \le c$, the following holds
\begin{equation*}
    \mathbb{P}\biggl(\,\Bigl\Vert \frac{1}{n}\sum_{i=1}^n z_i  \Bigr\Vert_2 \ge \epsilon\,\biggr) \, \le\, \exp\Bigl(-\frac{n\epsilon^2}{8c^2} +\frac{1}{4}\Bigr)
\end{equation*}

\end{lemma}

We apply the above Vector Bernstein Inequality to the random vectors $\phi(x_1)\otimes \phi(x_1) ,\ldots , \phi(x_1)\otimes \phi(x_1)$ where $\otimes$ denotes the Kronecker product. To do this, we define vector $v_i = \phi(x_i)\otimes \phi(x_i) - \mathbb{E}_{x\sim P }\bigl[\phi(x)\otimes \phi(x)\bigr]$ for every $i$. Note that $v_i$ is, by definition, a zero-mean vector and also for every $x$ we have the following for the normalized kernel function $k$:
\begin{equation*}
    \bigl\Vert \phi(x) \otimes \phi(x)\bigr\Vert^2_2 = \bigl\Vert \phi(x) \bigr\Vert^2_2 \cdot \bigl\Vert\phi(x)\bigr\Vert^2_2 = k(x,x)\cdot k(x,x) =1
\end{equation*}
Then, the triangle inequality implies that $$\bigl\Vert v_i\bigr\Vert_2 \le \bigl\Vert \phi(x_i)\otimes \phi(x_i)\bigr\Vert_2 + \bigl\Vert \mathbb{E}_{x\sim P }\bigl[\phi(x)\otimes \phi(x)\bigr] \bigr\Vert_2 \le \bigl\Vert \phi(x_i)\otimes \phi(x_i)\bigr\Vert_2 +  \mathbb{E}_{x\sim P }\bigl[\bigl\Vert\phi(x)\otimes \phi(x)\bigr\Vert_2\bigr]   = 2$$ 
As a result, the Vector Bernstein Inequality leads to the following for every $0\le \epsilon\le 2$:
\begin{equation*}
    \mathbb{P}\Bigl(\Bigl\Vert \frac{1}{n}\sum_{i=1}^n \phi(x_i)\otimes \phi(x_i) - \mathbb{E}_{x\sim P }\bigl[\phi(x)\otimes \phi(x)\bigr]  \Bigr\Vert_2 \ge \epsilon\Bigr) \, \le\, \exp\Bigl(\frac{8-n\epsilon^2}{32}\Bigr)
\end{equation*}
On the other hand, note that $\phi(x)\otimes \phi(x)$ is the vectorized version of rank-1 $\phi(x) \phi(x)^\top$, which shows that the above inequality is equivalent to the following where $\Vert\cdot\Vert_{\mathrm{HS}}$ denotes the Hilbert-Schmidt norm, which will simplify to the Frobenius norm in the finite dimension case,
\begin{align*}
    &\mathbb{P}\Bigl(\Bigl\Vert \frac{1}{n}\sum_{i=1}^n  \bigl[\phi(x_i) \phi(x_i)^\top\bigr] - \mathbb{E}_{x\sim P }\bigl[\phi(x) \phi(x)^\top\bigr]  \Bigr\Vert_{\mathrm{HS}} \ge \epsilon\Bigr) \, \le\, \exp\Bigl(\frac{8-n\epsilon^2}{32}\Bigr) \\
    \Longrightarrow \;\; & \mathbb{P}\Bigl(\Bigl\Vert C_X - \widetilde{C}_X  \Bigr\Vert_{\mathrm{HS}} \ge \epsilon\Bigr) \, \le\, \exp\Bigl(\frac{8-n\epsilon^2}{32}\Bigr)
\end{align*}
Subsequently, we can apply the Hoffman-Wielandt inequality which shows that for the sorted eigenvalue vectors of $C_X$ (denoted by $\widehat{\boldsymbol{\lambda}}_n$ in the theorem) and $\widetilde{C}_X$ (denoted by $\widetilde{\boldsymbol{\lambda}}$ in the theorem) we will have
$\Vert \widehat{\boldsymbol{\lambda}}_n - \widetilde{\boldsymbol{\lambda}} \Vert_2\le \Vert C_X - \widetilde{C}_X  \Vert_{\mathrm{HS}}$, which together with the previous inequality leads to

\begin{align*}
    & \mathbb{P}\Bigl(\bigl\Vert \widehat{\boldsymbol{\lambda}}_n - \widetilde{\boldsymbol{\lambda}} \bigr\Vert_2 \ge \epsilon\Bigr) \, \le\, \exp\Bigl(\frac{8-n\epsilon^2}{32}\Bigr)
\end{align*}
If we define $\delta = \exp\bigl((8-n\epsilon^2)/32\bigr)$ that implies $\epsilon \le \sqrt{\frac{32\log(2/\delta)}{n}}$, we obtain the following for every $\delta \ge \exp\bigl((2-n)/8\bigr)$ (since we suppose $0\le \epsilon \le 2$)
\begin{align*}
    & \mathbb{P}\Bigl(\bigl\Vert \widehat{\boldsymbol{\lambda}}_n - \widetilde{\boldsymbol{\lambda}} \bigr\Vert_2 \ge \sqrt{\frac{32\log(2/\delta)}{n}}\Bigr) \, \le\, \delta \\
    \Longrightarrow \;\; & \mathbb{P}\Bigl(\bigl\Vert \widehat{\boldsymbol{\lambda}}_n - \widetilde{\boldsymbol{\lambda}} \bigr\Vert_2 \le \sqrt{\frac{32\log(2/\delta)}{n}}\Bigr) \, \ge\, 1- \delta
\end{align*}
which completes the proof.

\subsection{Proof of Corollary \ref{Corollary: Finite Dimension}}

\textbf{The case of $\alpha=1$}. We show that Theorem \ref{Thm: eigenvalue convergence} on the concentration of  the eigenvalues $\boldsymbol{\lambda}=[\lambda_1,\ldots , \lambda_d]$ will further imply a concentration bound for the logarithm of Vendi-1 score. In the case of $\mathrm{Vendi}_1$ (when $\alpha\rightarrow 1^+$), the concentration bound will be formed for the logarithm of the Vendi score, i.e. the Von-Neumann entropy (denoted as $H_{\alpha}$):

$$H_1(C_X):= H_1(\boldsymbol{\lambda}) = \sum_{i=1}^d \widetilde{\lambda}_i \log \frac{1}{\widetilde{\lambda}_i}$$

Theorem \ref{Thm: eigenvalue convergence} shows that $\Vert \widehat{\boldsymbol{\lambda}}_n - \widetilde{\boldsymbol{\lambda}} \Vert_2 \le \sqrt{\frac{32\log(2/\delta)}{n}}$ with probability $1-\delta$. To convert this concentration bound to a bound on the order-1 entropy (for Vendi-1 score) difference $H_1({\widehat{C}_n}) - H_1(C_X)$, we leverage the following two lemmas:

\begin{lemma}\label{Lemma: Log inequality}
For every $0\le \alpha ,\beta\le 1$ such that $|\beta-\alpha|\le \frac{1}{e}$, we have $$\Bigl\vert\alpha \log \frac{1}{\alpha} - \beta \log \frac{1}{\beta}\Bigr\vert \,\le\, \vert\beta -\alpha\vert\log\frac{1}{\vert\beta -\alpha\vert}$$ 
\end{lemma}
\begin{proof}
Let $c = |\alpha-\beta|$, where $c\in[0, \frac{1}{e}]$. Defining $g(z) = z\log(\frac{1}{z})$, the first-order optimality condition $g'(z)=-\log(z) - 1 = 0$ yields $\frac{1}{e}$ as the local maximum of $g(z)$. Therefore, there are three cases of placement of $\alpha$ and $\beta$ on the interval $[0,1]$: $\alpha$ and $\beta$ appear before maximum point, after maximum point or maximum point is between $\alpha$ and $\beta$. We show that regardless of the placement of $\alpha$ and $\beta$, the above inequality remains true.
\begin{itemize}
    \item \textbf{Case 1:} $\alpha, \beta \in [0, \frac{1}{e}]$. Note that $g''(z) = -\frac{1}{z}$. Since the second-order derivative is negative and the function $g$ is monotonically increasing within the interval $[0,\frac{1}{e}]$, the gap between $g(\alpha)$ and $g(\beta)$ is maximized when $\alpha^*=0$ and $\beta^*=c-\alpha^* = c$. This directly leads to the desired bound as follows:
    $$\Bigl\vert\alpha \log \frac{1}{\alpha} - \beta \log \frac{1}{\beta}\Bigr\vert \, \le\, \Bigl\vert\alpha^* \log \frac{1}{\alpha^*} - \beta^* \log \frac{1}{\beta^*}\Bigr\vert \, =\, \bigl\vert0 \log 0 - c \log \frac{1}{c}\bigr\vert \, \le\, c \log\frac{1}{c}$$ 
    Here, we use the standard limit $0\log 0 = 0$.
    \item \textbf{Case 2:} $\alpha, \beta \in [\frac{1}{e}, 1]$. In this case, we note that $g$ is concave yet decreasing over $[\frac{1}{e},1]$, and so the gap between $g(\alpha)$ and $g(\beta)$ will be maximized when $\alpha^*=1-c$ and $\beta^*=1$. This leads to:
    $$\Bigl\vert\alpha \log \frac{1}{\alpha} - \beta \log \frac{1}{\beta}\Bigr\vert \, \le\, \Bigl\vert\alpha^* \log \frac{1}{\alpha^*} - \beta^* \log \frac{1}{\beta^*}\Bigr\vert\, =\,   (1-c) \log \frac{1}{(1-c)} \le c\log\frac{1}{c}$$ 
    where the last inequality holds because $c\in[0, \frac{1}{e}]$, and if we define the function $h(c)= c\log\frac{1}{c} - (1-c)\log\frac{1}{1-c}$, then we have $h'(c) = \log\frac{1}{c(1-c)} - 2$, which is positive over $c\in[0,c_0]$ ($e^{-2}<c_0<e^{-1}$ is where $c_0(1-c_0)=e^{-2}$), and then negative over $[c_0,\frac{1}{e}]$, and hence $h(c)\ge \min\{h(0),h(1/e)\} = 0$ for every $c\in[0,1/e]$.  
    \item \textbf{Case 3:} $\alpha \in [0, \frac{1}{e})$ and $\beta \in (\frac{1}{e}, 1]$. When $\alpha$ and $\beta$ lie on the opposite ends from the maximum point, the inequality becomes:
    $$\bigl\vert\alpha \log \frac{1}{\alpha} - \beta \log \frac{1}{\beta}\bigr\vert \leq \mathrm{max} \Bigr\{ \bigl\vert(1/e)\log \frac{1}{1/e} - \beta\log\frac{1}{\beta}\bigr\vert , \bigr\vert \alpha\log\frac{1}{\alpha} - (1/e)\log \frac{1}{1/e}\bigl\vert\Bigl\} \leq c\log\frac{1}{c}$$
    since we pick the side with the largest difference, this difference is upper bounded by either Case 1 or Case 2 because $\mathrm{max}\{|\frac{1}{e}-\beta|, |\alpha-\frac{1}{e}|\} < c$. Therefore, this case is upper-bounded by $c\log\frac{1}{c}$.
\end{itemize}
All the three cases of placement of $\alpha$ and $\beta$ are upper-bounded by $c\log \frac{1}{c}$; Therefore, the claim holds.
\end{proof}

\begin{lemma}\label{Lemma: Entropy Schur-concave}
If $\Vert \mathbf{u}\Vert_2\le \epsilon$ for $d$-dimensional vector $\mathbf{u}\ge \mathbf{0}$ where $\epsilon\le\frac{1}{e}$, then we have $$\sum_{i=1}^d u_i \log\frac{1}{u_i} \le \epsilon\sqrt{d}\log\frac{\sqrt{d}}{\epsilon}$$ 
\end{lemma}
\begin{proof}
We prove the above inequality using the KKT conditions for the following maximization problem, representing a convex optimization problem,

\begin{align*}
    &\max_{\mathbf{u}\in\mathbb{R}^d}\qquad\;\;\: \sum_{i=1}^du_i\log(\frac{1}{u_i}) \\
    &\text{\rm subject to}\quad u_i \ge 0, \;\; \text{\rm for all}\: i\\
    &\qquad\qquad\quad \sum_{i=1}^du_i^2 \leq \epsilon^2 \;\;\ (\text{equivalent to } \Vert \mathbf{u}\Vert_2\le \epsilon)
\end{align*}

In a concave maximization problem subject to convex constraints, any point that satisfies the KKT conditions is guaranteed to be a global optimum. Let us pick the following solution $\mathbf{u}^* = \frac{\epsilon}{\sqrt{d}}\mathbf{1}$ and slack variables $\lambda^*=\frac{\sqrt{d}}{2\epsilon}\bigr(\log(\frac{\sqrt{d}}{\epsilon}) -1 \bigl)$, $\forall_i \;\mu_i^*=0$. The Lagrangian of the above problem:

$$L(\mathbf{u},\lambda,\mu_1,\dots,\mu_d) = \sum_{i=1}^du_i\log(\frac{1}{u_i}) + \lambda(\epsilon^2-\sum_{i=1}^du_i^2)-\sum_{i=1}^d\mu_iu_i$$

\begin{itemize}
    \item \textbf{Primal Feasibility.} The solution $\mathbf{u^*}$ satisfies the primal feasibility, since $\epsilon^2 - \sum_{i=1}^d(\frac{\epsilon}{\sqrt{d}})^2 = \epsilon^2 - d\frac{\epsilon^2}{d} = 0$ and $\frac{\epsilon}{\sqrt{d}} \ge 0$.
    \item \textbf{Dual Feasibility.} $\lambda^*\geq0$ is feasible because of the assumption $\epsilon\le \frac{1}{e}$ implying that $\frac{\sqrt{d}}{\epsilon}\ge e$ for every integer dimension $d\ge 1$. Note that this implies $\lambda^*= \frac{\sqrt{d}}{2\epsilon}\bigr(\log(\frac{\sqrt{d}}{\epsilon}) -1 \bigl) \ge 0$.
    \item \textbf{Complementary Slackness.} Since $\lambda^*\bigr( \epsilon^2 - \sum_{i=1}^d(\frac{\epsilon}{\sqrt{d}})^2\bigl) = \lambda^* \cdot  0 = 0$, the condition is satisfied.
    \item \textbf{Stationarity.} The condition is satisfied as follows:
    \begin{align*}
        \frac{\partial}{\partial u_i}L(\mathbf{u^*}) &= -\log(u_i^*)-1-2\lambda^*u_i^* + \mu_i^* = -\log(\frac{\epsilon}{\sqrt{d}}) - 1 - 2\cdot\frac{\sqrt{d}}{2\epsilon}\bigr( -\log(\frac{\epsilon}{\sqrt{d}}) -1 \bigl)\cdot \frac{\epsilon}{\sqrt{d}} = 0
    \end{align*}
\end{itemize}

Since all KKT conditions are satisfied and sufficient for global optimality, $\mathbf{u}^* = \frac{\epsilon}{\sqrt{d}}\mathbf{1}$ is a global optimum of the specified concave maximization problem. We note that this result is also implied by the Schur-concavity property of entropy. Following this result, the specified objective is upper-bounded as follows:
$$\sum_{i=1}^d u_i \log\frac{1}{u_i} \le \epsilon\sqrt{d}\log\frac{\sqrt{d}}{\epsilon}$$ 
Therefore, the lemma's proof is complete.
\end{proof}

Following the above lemmas, knowing that $\Vert \widehat{\lambda}_n - \widetilde{\lambda} \Vert_2 \le \sqrt{\frac{32\log(2/\delta)}{n}}$ from Theorem~\ref{Thm: eigenvalue convergence} and using the assumption $n\ge 32e^2\log(2/\delta)\approx 236.5 \log(2/\delta)$ that ensures the upper-bound satisfies $\sqrt{\frac{32\log(2/\delta)}{n}}\le \frac{1}{e}$, we can apply the above two lemmas to show that with probability $1-\delta$:

$$\Bigl\vert H_1({\widehat{C}_n}) - H_1(C_X)\Bigr\vert = \Bigl\vert H_1(\widehat{\lambda}_n) - H_1(\widetilde{\lambda})\Bigr\vert \le \sqrt{\frac{8 d \log(2/\delta)}{n}} \log\Bigl( \frac{nd}{32\log(2/\delta)}\Bigr)$$

Note that under a kernel function with finite dimension $d$, the above bound will be $\mathcal{O}\Bigl(\sqrt{\frac{d}{n}} \log\bigl(nd\bigr)\Bigr)$. 


\textbf{The case of $1< \alpha <2$.} Note that the inequality $\Vert v \Vert_\alpha \le d^{\frac{2-\alpha}{2}} \Vert v\Vert_2$ holds for every $d$-dimensional vector $v\in\mathbb{R}^d$. Therefore, we can repeat the proof of Corollary~\ref{Corollary: Order greater than 2} to show the following for every $1< \alpha <2$
\begin{align*}
    \Bigl\vert \mathrm{Vendi}_\alpha(x_1,\ldots ,x_n)^{\frac{1-\alpha}{\alpha}} -  \mathrm{Vendi}_\alpha(P_x)^{\frac{1-\alpha}{\alpha}}\Bigr\vert \, &=\, \Bigl\vert \bigl\Vert \widehat{\boldsymbol{\lambda}}_n\bigr\Vert_{\alpha} - \bigl\Vert \widetilde{\boldsymbol{\lambda}}\bigr\Vert_{\alpha}\Bigr\vert \\
    &\le\, \bigl\Vert \widehat{\boldsymbol{\lambda}}_n - \widetilde{\boldsymbol{\lambda}}\bigr\Vert_{\alpha} \\
    &\le\, d^{\frac{2-\alpha}{2}}\bigl\Vert \widehat{\boldsymbol{\lambda}}_n - \widetilde{\boldsymbol{\lambda}}\bigr\Vert_{2}.
\end{align*}
Consequently, Theorem~\ref{Thm: eigenvalue convergence} impies that for every $1\le \alpha<2 $ and $\delta\ge \exp((2-n)/8)$, the following holds with probability at least $1-\delta$
\begin{equation*}
     \Bigl\vert \mathrm{Vendi}_\alpha(x_1,\ldots ,x_n)^{\frac{1-\alpha}{\alpha}} -  \mathrm{Vendi}_\alpha(P_x)^{\frac{1-\alpha}{\alpha}}\Bigr\vert \, \le \,  d^{\frac{2-\alpha}{2}}\sqrt{\frac{32\log(2/\delta)}{n}} \, =\, \sqrt{\frac{32 d^{2-\alpha}\log(2/\delta)}{n}}
\end{equation*}

\subsection{Proof of Corollary \ref{Corollary: Order greater than 2}}
Considering the $\alpha$-norm definition $\Vert \mathbf{v}\Vert_\alpha= \bigl(\sum_{i=1}^d |v_i|^\alpha\bigr)^{1/\alpha}$, we can rewrite the order-$\alpha$ Vendi definition  as
\begin{equation*}
    \mathrm{Vendi}_\alpha(x_1,\ldots ,x_n) = \bigl\Vert \widehat{\boldsymbol{\lambda}}_n\bigr\Vert^{\frac{\alpha}{1-\alpha}}_{\alpha} \quad \Longleftrightarrow \quad \mathrm{Vendi}_\alpha(x_1,\ldots ,x_n)^{\frac{1-\alpha}{\alpha}} = \bigl\Vert \widehat{\boldsymbol{\lambda}}_n\bigr\Vert_{\alpha}
\end{equation*}
where $\widehat{\boldsymbol{\lambda}}_n$ is defined in Theorem~\ref{Thm: eigenvalue convergence}. Similarly, given the definition of $\widetilde{\boldsymbol{\lambda}}$ we can write 
\begin{equation*}
    \mathrm{Vendi}_\alpha(P_x)^{\frac{1-\alpha}{\alpha}} = \bigl\Vert \widetilde{\boldsymbol{\lambda}}\bigr\Vert_{\alpha}
\end{equation*}
Therefore, for every $\alpha\ge 2$, the following hold due to the triangle inequality: 
\begin{align*}
    \Bigl\vert \mathrm{Vendi}_\alpha(x_1,\ldots ,x_n)^{\frac{1-\alpha}{\alpha}}  -  \mathrm{Vendi}_\alpha(P_x)^{\frac{1-\alpha}{\alpha}}\Bigr\vert \, &=\, \Bigl\vert \bigl\Vert \widehat{\boldsymbol{\lambda}}_n\bigr\Vert_{\alpha} - \bigl\Vert \widetilde{\boldsymbol{\lambda}}\bigr\Vert_{\alpha}\Bigr\vert \\
    &\le\, \bigl\Vert \widehat{\boldsymbol{\lambda}}_n - \widetilde{\boldsymbol{\lambda}}\bigr\Vert_{\alpha} \\
    &\le\, \bigl\Vert \widehat{\boldsymbol{\lambda}}_n - \widetilde{\boldsymbol{\lambda}}\bigr\Vert_{2}.
\end{align*}
As a result, Theorem~\ref{Thm: eigenvalue convergence} shows that for every $\alpha\ge 2$ and $\delta\ge \exp((2-n)/8)$, the following holds with probability at least $1-\delta$
\begin{equation*}
     \Bigl\vert \mathrm{Vendi}_\alpha(x_1,\ldots ,x_n)^{\frac{1-\alpha}{\alpha}} -  \mathrm{Vendi}_\alpha(P_x)^{\frac{1-\alpha}{\alpha}}\Bigr\vert \, \le \, \sqrt{\frac{32\log(2/\delta)}{n}}
\end{equation*}

\subsection{Proof of Theorem~\ref{Thm: truncated Vendi guarantee}}
We begin by proving the following lemma showing that the eigenvalues used in the definition of the $t$-truncated Vendi score are the projection of the original eigenvalues onto a $t$-dimensional probability simplex.
\begin{lemma}\label{Lemma: Projection}
   Consider $\mathbf{v}\in [0,1]^d$ that satisfies $\mathbf{1}^\top \mathbf{v} = 1$. i.e., the sum of $\mathbf{v}$'s entries equals $1$. Given integer $1\le t\le d$, define vector $\mathbf{v}^{(t)} \in [0,1]^d$ whose last $d-t$ entries are $0$, i.e., $v^{(t)}_{i} =0 $ for $t+1\le i\le d$, and its first $t$ entries are defined as $v^{(t)}_{j} = v_j + \frac{1-S_t}{t}$ where $S_t = v_1+\cdots +v_t$. Then, $\mathbf{v}^{(t)}$ is the projection of $\mathbf{v}$ onto the following simplex set and has the minimum $\ell_2$-norm distance to this set
   \begin{equation*}
       \Delta_t := \Bigl\{ \mathbf{u}\in [0,1]^d:\; v_i = 0\: \text{\rm for all}\: t+1\le i\le d,\;\;\; \sum_{i=1}^t v_i = 1  \Bigr\}.
   \end{equation*}
\end{lemma}
\begin{proof}
To prove the lemma, first note that $\mathbf{v}^{(t)} \in \Delta_t$, i.e. its first $t$ entries are non-negative and add up to $1$, and also its last $d-t$ entries are zero. Then, consider the projection problem discussed in the lemma:
\begin{align*}
    &\min_{\mathbf{u}\in\mathbb{R}^t}\qquad\;\;\: \sum_{i=1}^t \bigl(u_i - v_i\bigr)^2 \\
    &\text{\rm subject to}\quad u_i\ge 0, \;\; \text{\rm for all}\: i\\
    &\qquad\qquad\quad \sum_{i=1}^t u_i = 1
\end{align*}
Then, since we know from the assumptions that $v_i\ge 0$ and $\sum_{i=1}^t v_i \le 1$, the discussed $\mathbf{u}^*\in\mathbb{R}^t$ where $u^*_i = v_i + (1-S_t)/t$ together with Lagrangian coefficients $\mu_i =0 $ (for inequality constraint $u_i\ge 0$) and $\lambda = (1 - S_t)/t$ (for equality constraint) satisfy the KKT conditions. The primal and dual feasibility conditions as well as the complementary slackness clearly hold for these selection of primal and dual variables. Also, the KKT stationarity condition is satisfied as for every $i$ we have $u^*_i -v_i -\lambda - \mu_i = 0$. Since the optimization problem is a convex optimization task with affine constraints, the KKT conditions are sufficient for optimaility which proves the lemma.  
\end{proof}

Based on the above lemma, the eigenvalues $\widehat{\boldsymbol{\lambda}}_n^{(t)}$ used to calculate the $t$-truncated Vendi score $\mathrm{Vendi}^{(t)}_\alpha(x_1,\ldots ,x_n)$ are the projections of the top-$t$ eigenvalues in $\widehat{\boldsymbol{\lambda}}_n$ for the original score $\mathrm{Vendi}_\alpha(x_1,\ldots ,x_n)$ onto the $t$-simplex subset of $\mathbb{R}^d$ according to the $\ell_2$-norm. Similarly, the eigenvalues $\widetilde{\boldsymbol{\lambda}}_n^{(t)}$ used to calculate the $t$-truncated population Vendi $\mathrm{Vendi}^{(t)}_\alpha(P_X)$ are the projections of the top-$t$ eigenvalues in $\widetilde{\boldsymbol{\lambda}}$ for the original population Vendi $\mathrm{Vendi}_\alpha(P_x)$ onto the $t$-simplex subset of $\mathbb{R}^d$.

Since $\ell_2$-norm is a Hilbert space norm and the $t$-simplex subset $\Delta_t$ is a convex set, we know from the convex analysis that the $\ell_2$-distance between the projected points $\widehat{\boldsymbol{\lambda}}_n^{(t)}$ and $\widetilde{\boldsymbol{\lambda}}^{(t)}$ is upper-bounded by the $\ell_2$-distance between the original points $\widehat{\boldsymbol{\lambda}}_n$ and $\widetilde{\boldsymbol{\lambda}}$. As  a result, Theorem~\ref{Thm: eigenvalue convergence} implies that 
\begin{align*}
    & \mathbb{P}\Bigl(\bigl\Vert \widehat{\boldsymbol{\lambda}}_n - \widetilde{\boldsymbol{\lambda}} \bigr\Vert_2 \le \sqrt{\frac{32\log(2/\delta)}{n}}\Bigr) \, \ge\, 1- \delta \\
    \Longrightarrow \;\; & \mathbb{P}\Bigl(\bigl\Vert \widehat{\boldsymbol{\lambda}}^{(t)}_n - \widetilde{\boldsymbol{\lambda}}^{(t)} \bigr\Vert_2 \le \sqrt{\frac{32\log(2/\delta)}{n}}\Bigr) \, \ge\, 1- \delta
\end{align*}
However, note that the eigenvalue vectors $\widehat{\boldsymbol{\lambda}}^{(t)}_n$ and $\widetilde{\boldsymbol{\lambda}}^{(t)}$ can be analyzed in a bounded $t$-dimensional space as their entries after index $t+1$ are zero. Therefore, we can apply the proof of Corollary~\ref{Corollary: Finite Dimension} to show that for every $1\le \alpha<2 $ and $\delta\ge \exp((2-n)/8)$, the following holds with probability at least $1-\delta$
\begin{equation*}
     \Bigl\vert \mathrm{Vendi}_\alpha(x_1,\ldots ,x_n)^{\frac{1-\alpha}{\alpha}} -  \mathrm{Vendi}_\alpha(P_x)^{\frac{1-\alpha}{\alpha}}\Bigr\vert \, \le \,   \sqrt{\frac{32 t^{2-\alpha}\log(2/\delta)}{n}}
\end{equation*}
To extend the result to a general $\alpha>1$, we reach the following inequality covering the above result as well as the result of Corollary~\ref{Corollary: Order greater than 2} in one inequality
\begin{equation*}
     \Bigl\vert \mathrm{Vendi}_\alpha(x_1,\ldots ,x_n)^{\frac{1-\alpha}{\alpha}} -  \mathrm{Vendi}_\alpha(P_x)^{\frac{1-\alpha}{\alpha}}\Bigr\vert \, \le \, \sqrt{\frac{32 \max\{1, t^{2-\alpha}\}\log(2/\delta)}{n}}
\end{equation*}

\subsection{Proof of Theorem~\ref{Thm: Nyström, FKEA}}
\textbf{Proof of Part (a)}. As defined by \cite{ospanov_fkea_2024},  the FKEA method uses the eigenvalues of $t$ random Fourier frequencies $\omega_1,\ldots , \omega_{t}$ where for each $\omega_i$ they consider two features $\cos(\omega_i^\top x)$ and $\sin(\omega_i^\top x)$. Following the definitions, it can be seen that $k(x,x') = \mathbb{E}_{\omega\sim p_\omega}\bigl[\cos(\omega^{\top}(x-x'))\bigr]$ which is approximated by FKEA as $\frac{1}{t}\sum_{i=1}^{t} \cos(\omega_i^{\top}(x-x'))$. Therefore, if we define kernel matrix $K_i$ as the kernel matrix for $k_i(x,x') = \cos(\omega_i^{\top}(x-x'))$, then we will have
\begin{equation*}
    \frac{1}{n} K^{\mathrm{FKEA}(t)} = \frac{1}{t}\sum_{i=1}^{t} \frac{1}{n}K_i
\end{equation*}
where $\mathbb{E}_{\omega_i\sim p_\omega}[\frac{1}{n}K_i] = \frac{1}{n}K$.

On the other hand, we note that $\Vert \frac{1}{n} K \Vert_{\mathrm{HS}} \le 1$ holds as the kernel function is normalized and hence $|k(x,x')|\le 1$. Since the Frobenius norm is the $\ell_2$-norm of the vectorized version of the matrix, we can apply Vector Bernstein inequality in Lemma~\ref{Lemma: Vector Bernstein} to show that for every $0\le \epsilon \le 2$: 
\begin{align*}
    &\mathbb{P}\Bigl(\Bigl\Vert \frac{1}{t}\sum_{i=1}^t  \bigl[\frac{1}{n}K_i\bigr] - \frac{1}{n}K  \Bigr\Vert_{F} \ge \epsilon\Bigr) \, \le\, \exp\Bigl(\frac{8-t\epsilon^2}{32}\Bigr) \\
    \Longrightarrow \;\; & \mathbb{P}\Bigl(\Bigl\Vert \frac{1}{n}K^{\mathrm{FKEA}(t)} - \frac{1}{n}K  \Bigr\Vert_{F} \ge \epsilon\Bigr) \, \le\, \exp\Bigl(\frac{8-t\epsilon^2}{32}\Bigr)
\end{align*}
Then, we apply the Hoffman-Wielandt inequality to show that for the sorted eigenvalue vectors of $\frac{1}{n}K$ (denoted by $\widehat{\boldsymbol{\lambda}}_n$) and $\frac{1}{n}K^{\mathrm{FKEA}(t)}$ (denoted by ${\boldsymbol{\lambda}}^{\mathrm{FKEA}(t)}$) we will have
$\Vert \widehat{\boldsymbol{\lambda}}_n - {\boldsymbol{\lambda}}^{\mathrm{FKEA}(t)} \Vert_2\le \Vert \frac{1}{n}K^{\mathrm{FKEA}(t)} - \frac{1}{n}K \Vert_{\mathrm{HS}}$, which together with the previous inequality leads to
\begin{align*}
 \mathbb{P}\Bigl(\Bigl\Vert \widehat{\boldsymbol{\lambda}}_n - {\boldsymbol{\lambda}}^{\mathrm{FKEA}(t)} \Bigr\Vert_2 \ge \epsilon\Bigr) \, \le\, \exp\Bigl(\frac{8-t\epsilon^2}{32}\Bigr)
\end{align*}
Furthermore, as we shown in the proof of Theorem~\ref{Thm: eigenvalue convergence} for every $0\le \gamma\le 2$
\begin{align*}
    & \mathbb{P}\Bigl(\bigl\Vert \widehat{\boldsymbol{\lambda}}_n - \widetilde{\boldsymbol{\lambda}} \bigr\Vert_2 \ge \gamma\Bigr) \, \le\, \exp\Bigl(\frac{8-n\gamma^2}{32}\Bigr)
\end{align*}
which, by applying the union bound for $\gamma = \epsilon$, together with the previous inequality shows that
\begin{align*}
 \mathbb{P}\Bigl(\Bigl\Vert \widetilde{\boldsymbol{\lambda}} - {\boldsymbol{\lambda}}^{\mathrm{FKEA}(t)} \Bigr\Vert_2 \ge 2\epsilon\Bigr) \, &\le\, \exp\Bigl(\frac{8-t\epsilon^2}{32}\Bigr)+\exp\Bigl(\frac{8-n\epsilon^2}{32}\Bigr) \\
 &\le \, 2\exp\Bigl(\frac{8-\min\{n,t\}\epsilon^2}{32}\Bigr)
\end{align*}
Therefore, Lemma~\ref{Lemma: Projection} implies that
\begin{align*}
 \mathbb{P}\Bigl(\Bigl\Vert \widetilde{\boldsymbol{\lambda}}^{(t)} - {\boldsymbol{\lambda}}^{\mathrm{FKEA}(t)} \Bigr\Vert_2 \ge \epsilon\Bigr) \, &\le \, 2\exp\Bigl(\frac{32-\min\{n,t\}\epsilon^2}{128}\Bigr)
\end{align*}
If we define $\delta= 2\exp\bigl(\frac{32-\min\{n,t\}\epsilon^2}{128}\bigr)$, implying that $\epsilon \le \sqrt{\frac{128\log(3/\delta)}{\min\{n,t\}}}$, then the above inequality shows that
\begin{align*}
 \mathbb{P}\Bigl(\Bigl\Vert \widetilde{\boldsymbol{\lambda}}^{(t)} - {\boldsymbol{\lambda}}^{\mathrm{FKEA}(t)} \Bigr\Vert_2 \le \sqrt{\frac{128\log(3/\delta)}{\min\{n,t\}}}\Bigr) \, &\ge 1-\delta
\end{align*}
Therefore, if we follow the same steps of the proof of Theorem~\ref{Thm: truncated Vendi guarantee}, we can show
\begin{equation*}
     \Bigl\vert \mathrm{FKEA}\text{-}\mathrm{Vendi}^{(t)}_\alpha(x_1,\ldots ,x_n)^{\frac{1-\alpha}{\alpha}} -  \mathrm{Vendi}^{(t)}_\alpha(P_x)^{\frac{1-\alpha}{\alpha}}\Bigr\vert \, \le \, \sqrt{\frac{128 \max\{1, t^{2-\alpha}\}\log(3/\delta)}{\min\{n,t\}}}
\end{equation*}

\textbf{Proof of Part (b)}. To show this theorem, we use Theorem~3 from \citep{xu2015nystrom}, which shows that if the $r$th largest eigenvalue of the kernel matrix $\frac{1}{n}K$ satisfies $\lambda_r \le \frac{\tau}{n}$, then given $t\ge Cr\log(n)$ ($C$ is a universal constant), the following spectral norm bound will hold with probability $1-\frac{2}{n^3}$:
\begin{equation*}
    \bigl\Vert \frac{1}{n}K - \frac{1}{n}K^{\mathrm{Nystrom(t)}} \bigr\Vert_{sp} \le \mathcal{O}\Bigl( \frac{\tau \log(n)}{\sqrt{nt}}\Bigr).
\end{equation*}
Therefore, Weyl's inequality implies the following for the vector of sorted eigenvalues of $\frac{1}{n}K$, i.e. $\widehat{\boldsymbol{\lambda}}_n$, and that of $\frac{1}{n}K^{\mathrm{Nystrom(t)}}$, i.e., ${\boldsymbol{\lambda}}^{\mathrm{Nystrom(t)}}$,
\begin{equation*}
    \bigl\Vert \widehat{\boldsymbol{\lambda}}_n - {\boldsymbol{\lambda}}^{\mathrm{Nystrom(t)}} \bigr\Vert_{\infty} \le \mathcal{O}\Bigl( \frac{\tau \log(n)}{\sqrt{nt}}\Bigr).
\end{equation*}
As a result, considering the subvectors $\widehat{\boldsymbol{\lambda}}_n[1:t]$ and ${\boldsymbol{\lambda}}^{\mathrm{Nystrom(t)}}[1:t]$ with the first $t$ entries of the vectors, we will have:
\begin{equation*}
    \bigl\Vert \widehat{\boldsymbol{\lambda}}_n[1:t] - {\boldsymbol{\lambda}}^{\mathrm{Nystrom(t)}}[1:t] \bigr\Vert_{\infty} \le \mathcal{O}\Bigl( \frac{\tau \log(n)}{\sqrt{nt}}\Bigr) \quad \Longrightarrow \quad \bigl\Vert \widehat{\boldsymbol{\lambda}}_n[1:t] - {\boldsymbol{\lambda}}^{\mathrm{Nystrom(t)}}[1:t] \bigr\Vert_{2} \le \mathcal{O}\Bigl( \tau \log(n)\sqrt{\frac{t}{n}}\Bigr)  
\end{equation*}
Noting that the non-zero entries of ${\boldsymbol{\lambda}}^{\mathrm{Nystrom(t)}}$ are all included in the first-$t$ elements, we can apply Lemma~\ref{Lemma: Projection} which shows that with probability $1-2n^{-3}$ we have
\begin{equation*}
     \Bigl\Vert \widehat{\boldsymbol{\lambda}}^{(t)}_n - {\boldsymbol{\lambda}}^{\mathrm{Nystrom(t)}}\Bigr\Vert_{2} \le \mathcal{O}\Bigl( \tau \log(n)\sqrt{\frac{t}{n}}\Bigr)  
\end{equation*}
Also, in the proof of Theorem~2, we showed that
\begin{align*}
    \mathbb{P}\Bigl(\bigl\Vert \widehat{\boldsymbol{\lambda}}^{(t)}_n - \widetilde{\boldsymbol{\lambda}}^{(t)} \bigr\Vert_2 \le \sqrt{\frac{32\log(2/\delta)}{n}}\Bigr) \, \ge\, 1- \delta
\end{align*}
Combining the above inequalities using a union bound, shows that with probability at least $1-\delta-2n^{-3}$ we have 
\begin{align*}
     \Bigl\Vert {\boldsymbol{\lambda}}^{\mathrm{Nystrom(t)}} - \widetilde{\boldsymbol{\lambda}}^{(t)} \Bigr\Vert_{2} \:&\le\:  \Bigl\Vert {\boldsymbol{\lambda}}^{\mathrm{Nystrom(t)}} - \widehat{\boldsymbol{\lambda}}^{(t)}_n \Bigr\Vert_{2} + \Bigl\Vert \widehat{\boldsymbol{\lambda}}^{(t)}_n - \widetilde{\boldsymbol{\lambda}}^{(t)} \Bigr\Vert_{2} \\
     \:&\le\:  \sqrt{\frac{32\log(2/\delta)}{n}}+ \mathcal{O}\Bigl( \tau \log(n)\sqrt{\frac{t}{n}}\Bigr) \\
     &=\: \mathcal{O}\Bigl(\sqrt{\frac{\log(2/\delta)+t\log(n)^2\tau^2}{n}}\Bigr)
\end{align*}
Hence, repeating the final steps in the proof of Theorem~\ref{Thm: truncated Vendi guarantee}, we can prove
\begin{equation*}
     \Bigl\vert \mathrm{Nystrom}\text{-}\mathrm{Vendi}^{(t)}_\alpha(x_1,\ldots ,x_n)^{\frac{1-\alpha}{\alpha}} -  \mathrm{Vendi}^{(t)}_\alpha(P_x)^{\frac{1-\alpha}{\alpha}}\Bigr\vert \, \le \, \mathcal{O}\Bigl(\sqrt{\frac{\max\{t^{2-\alpha},1\}\bigl(\log(2/\delta)+t\log(n)^2\tau^2\bigr)}{n}}\Bigr)
\end{equation*}

\section{Additional Numerical Results}
\begin{figure}
    \centering
    \includegraphics[width=\textwidth]{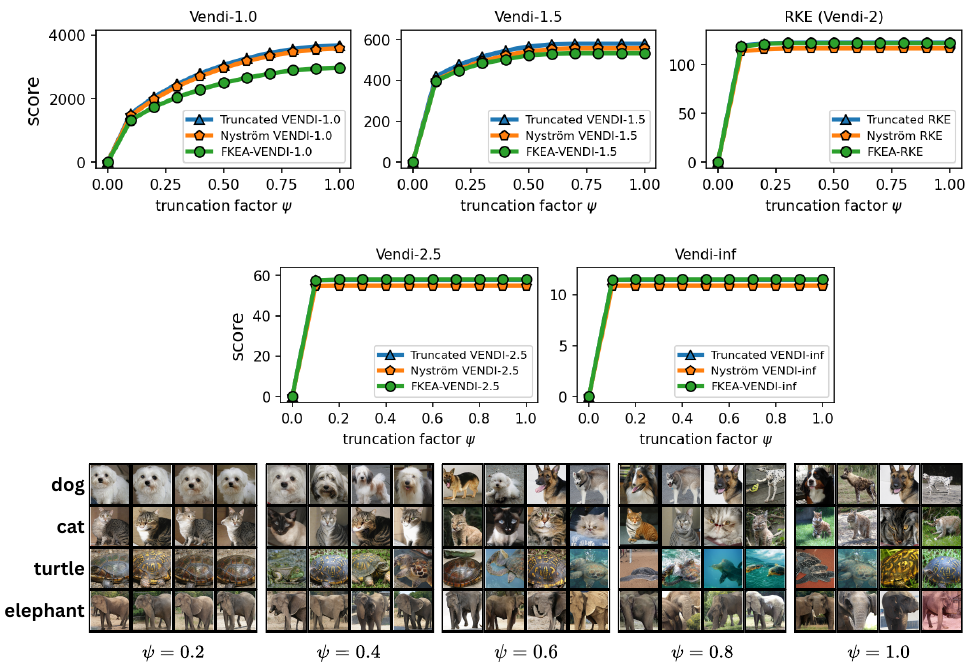}
    \caption{Diversity evaluation of Vendi scores on StyleGAN-XL generated ImageNet dataset with varying truncation parameter $\psi$. The setting is based on \textit{DinoV2} embedding and bandwidth $\sigma=30$}
  \label{VENDI_imagenet_diversity}
\end{figure}

In this section, we present supplementary results concerning the evaluation of diversity and the convergence behavior of different variants of the Vendi score. We extend the convergence experiments discussed in the main text to include the truncated StyleGAN3-t FFHQ dataset (Figure \ref{VENDI_stylegan3_convergence}) and the StyleGAN-XL ImageNet dataset (Figure \ref{VENDI_styleganxl_convergence}). Furthermore, we demonstrate that the truncated Vendi statistic effectively captures the diversity characteristics across various data modalities. Specifically, we conducted similar experiments as shown in Figures \ref{VENDI_imagenet_truncation} and \ref{VENDI_ffhq_truncation} on text data (Figure \ref{VENDI_countries_diversity}) and video data (Figure \ref{VENDI_video_diversity}), showcasing the applicability of the metric across different domains.

\begin{figure}
    \centering
    \includegraphics[width=\linewidth]{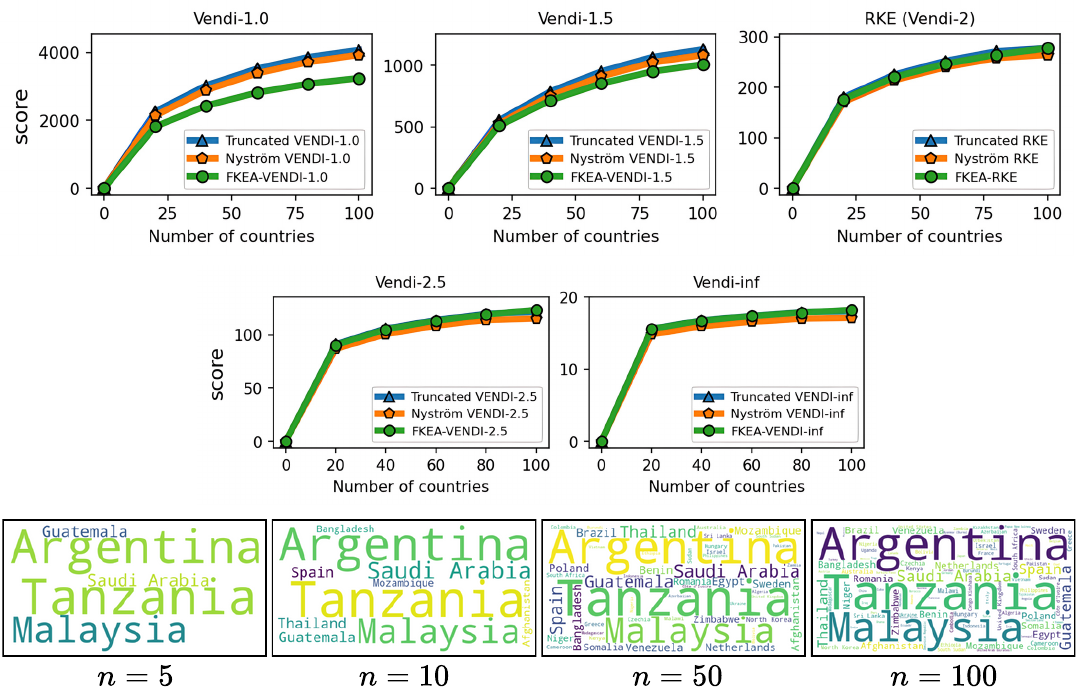}
    \caption{Diversity evaluation of Vendi scores on synthetic text dataset about 100 countries generated by GPT-4 with varying number of countries. The setting is based on \textit{text-embedding-3-large} embedding and bandwidth $\sigma=0.5$}
  \label{VENDI_countries_diversity}
\end{figure}

\begin{figure}
    \includegraphics[width=0.99\linewidth]{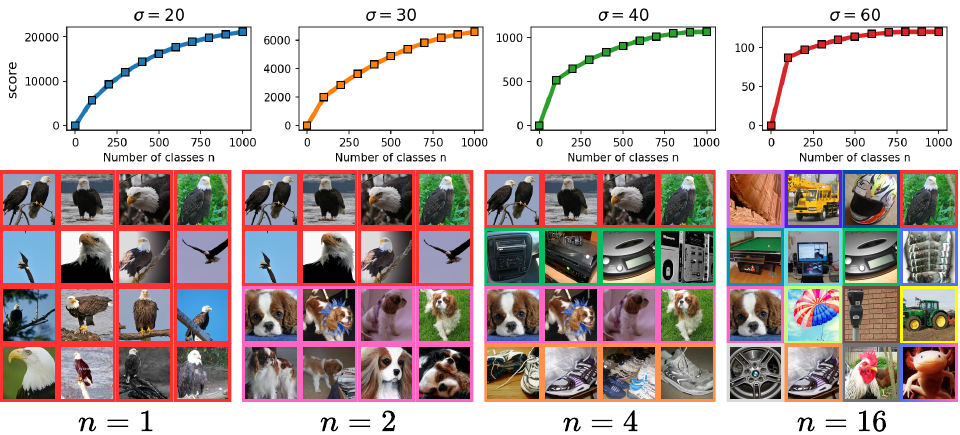}
    \caption{The diagram outlining an intuition behind a kernel bandwidth $\sigma$ selection in diversity evaluation.}
    \label{bandwidth_illustration}
\end{figure}

\begin{figure}
    \centering
    \includegraphics[width=\linewidth]{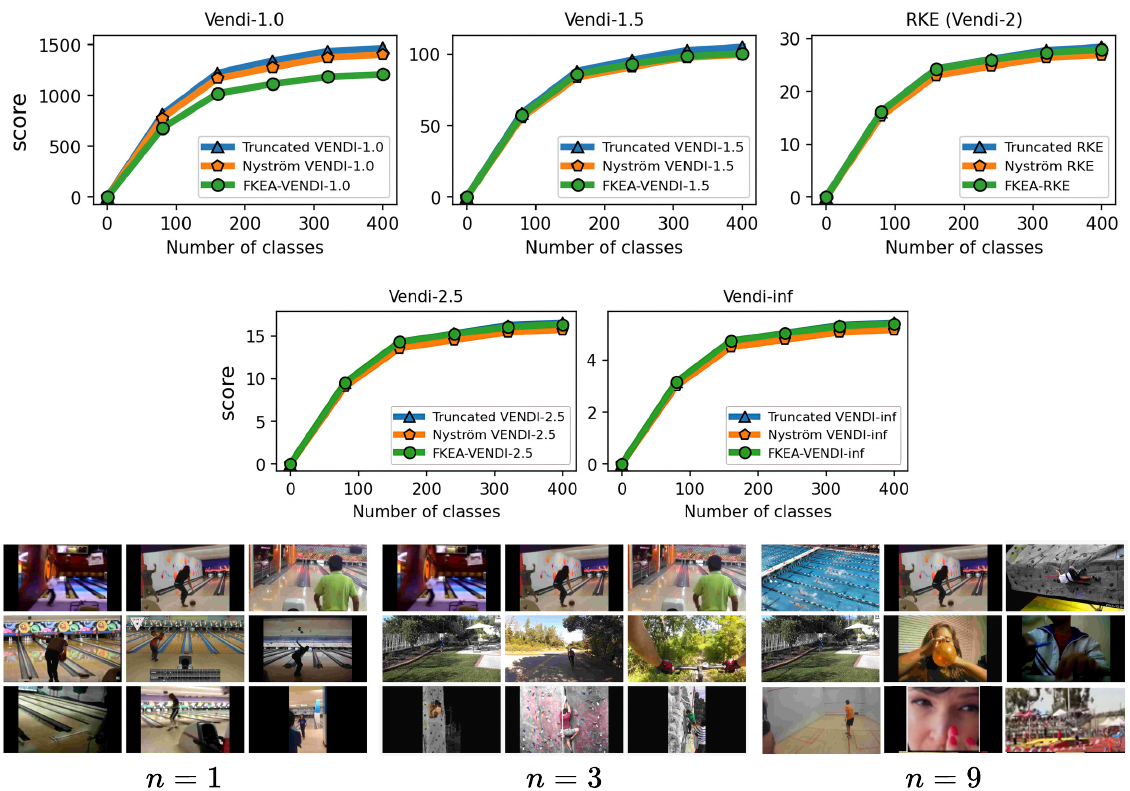}
    \caption{Diversity evaluation of Vendi scores on Kinetics400 dataset with varying number of classes. The setting is based on \textit{I3D} embedding and bandwidth $\sigma=4.0$}
  \label{VENDI_video_diversity}
\end{figure}

We observe in Figure \ref{VENDI_stylegan3_convergence} that the convergence behavior is illustrated across various values of $\psi$. The results indicate that, for a fixed bandwidth $\sigma$, the truncated, Nyström, and FKEA variants of the Vendi score converge to the truncated Vendi statistic. As demonstrated in Figure \ref{VENDI_ffhq_truncation} of the main text, this truncated Vendi statistic effectively captures the diversity characteristics inherent in the underlying dataset.

We note that in presence of incremental changes to the diversity of the dataset, finite-dimensional kernels, such as cosine similarity kernel, remain relatively constant. This effect is illustrated in Figure \ref{VENDI_styleganxl_convergence}, where increase in truncation factor $\psi$ results in incremental change in diversity. This is one of the cases where infinite-dimensional kernel maps with a sensitivity (bandwidth) parameter $\sigma$ are useful in controlling how responsive the method should be to the change in diversity.

\begin{figure}
    \centering
    \subfigure{\includegraphics[width=0.7\textwidth]{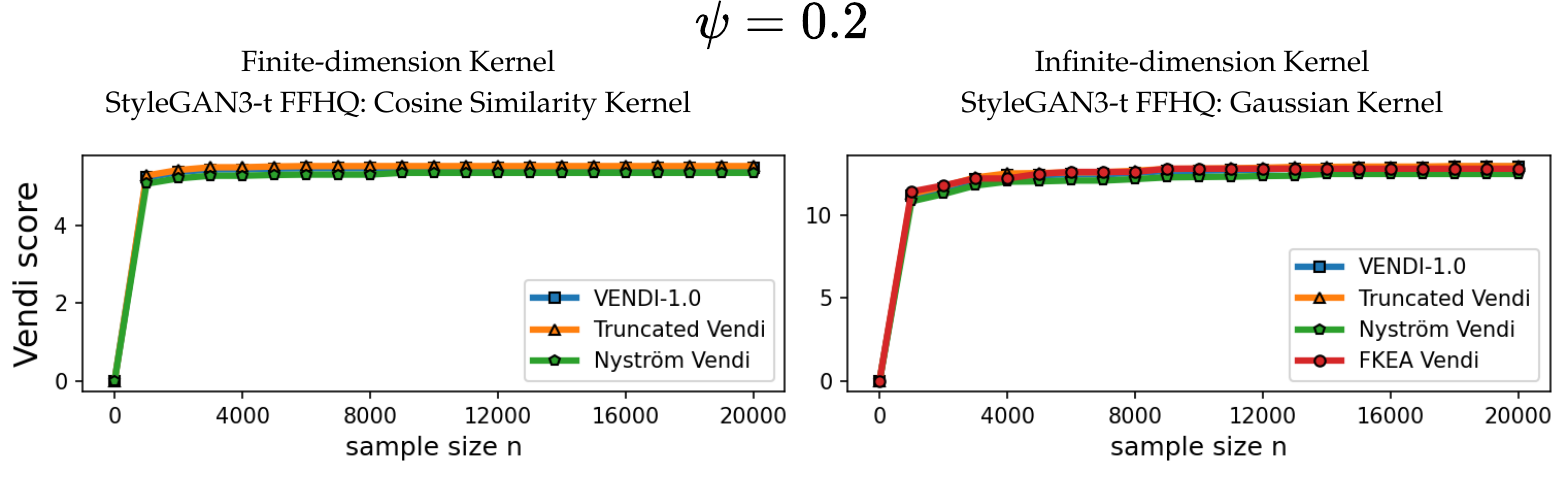}}
    \subfigure{\includegraphics[width=0.7\textwidth]{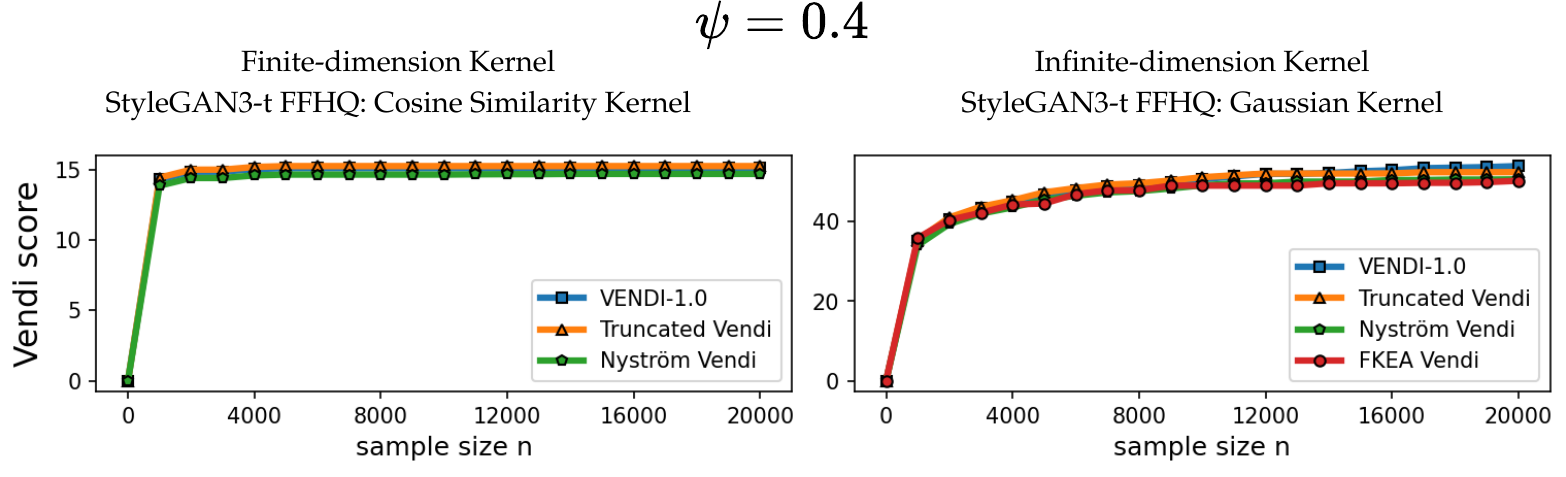}}
    \subfigure{\includegraphics[width=0.7\textwidth]{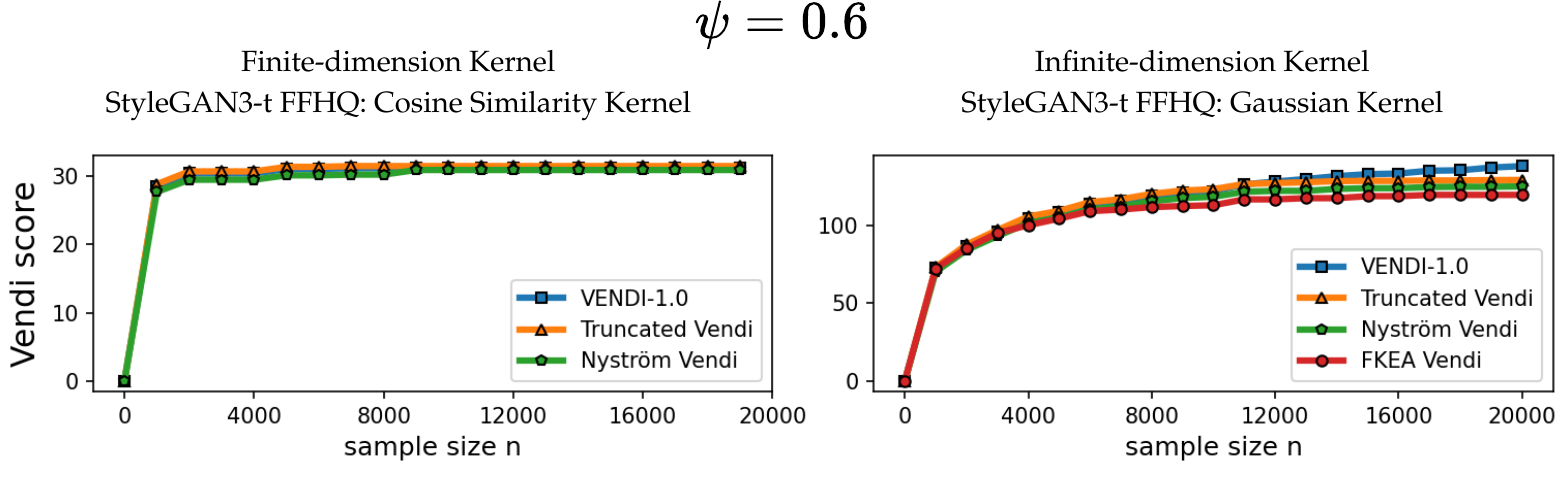}}
    \subfigure{\includegraphics[width=0.7\textwidth]{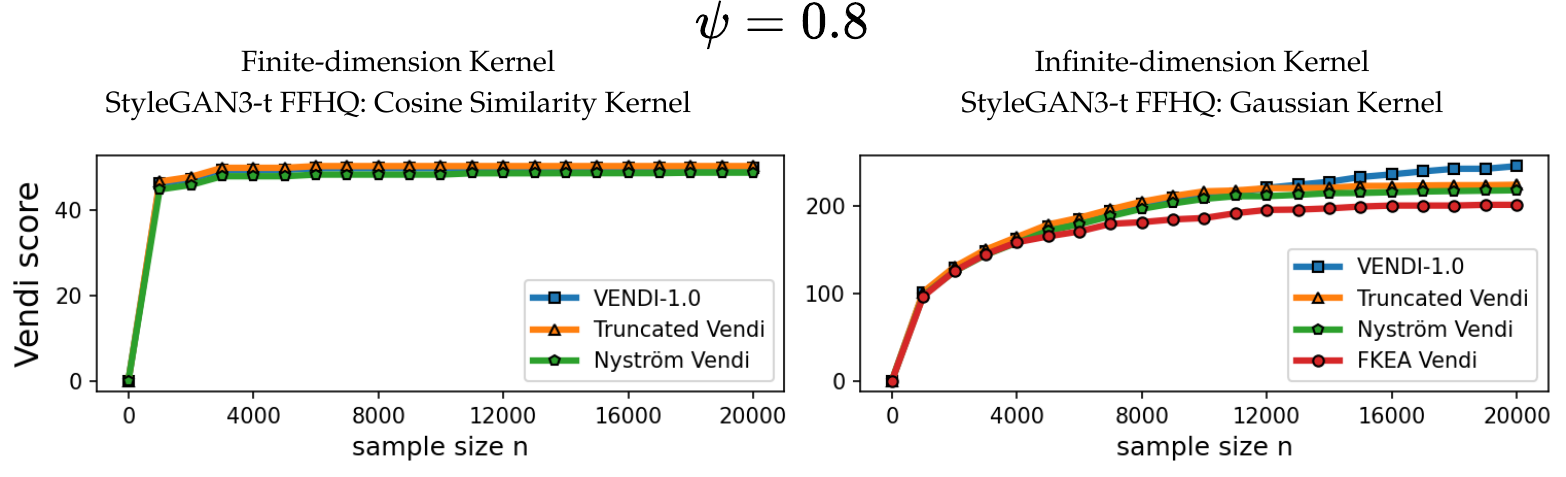}}
    \subfigure{\includegraphics[width=0.7\textwidth]{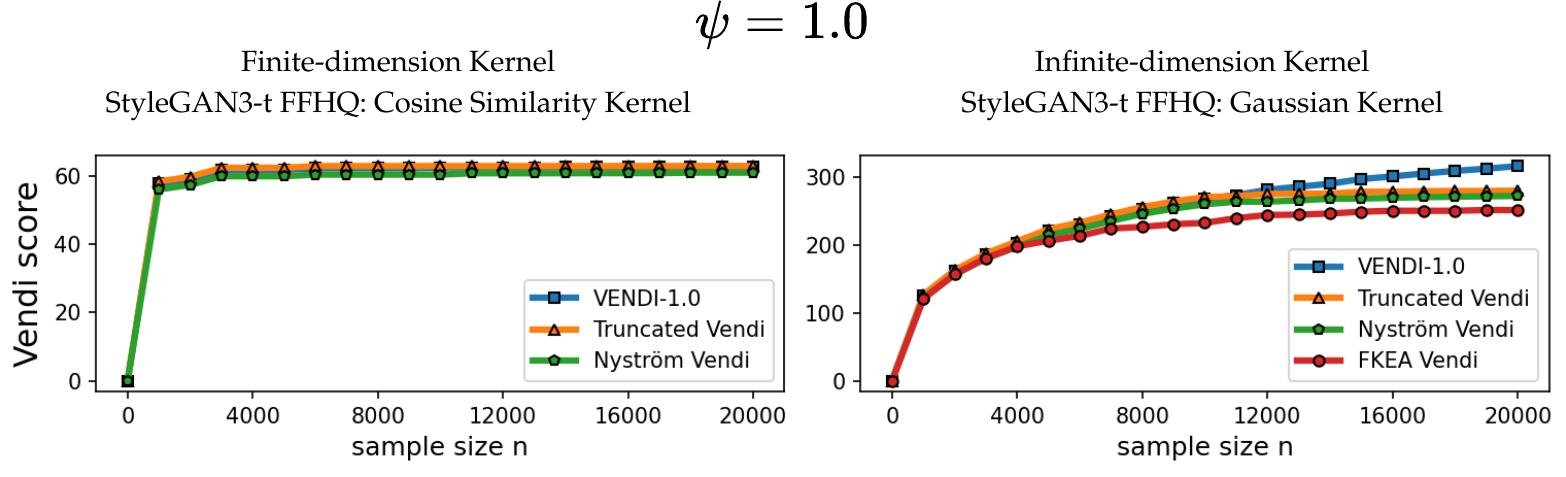}}
    \caption{Statistical convergence of Vendi score for different sample sizes on StyleGAN3 generated FFHQ data at various truncation factors $\psi$: (Left plot) finite-dimension cosine similarity kernel (Right plot) infinite dimension Gaussian kernel with bandwidth $\sigma=35$. \emph{DinoV2} embedding (dimension 768) is used in computing the scores.}
  \label{VENDI_stylegan3_convergence}
\end{figure}

\begin{figure}
    \centering
    \subfigure{\includegraphics[width=0.7\textwidth]{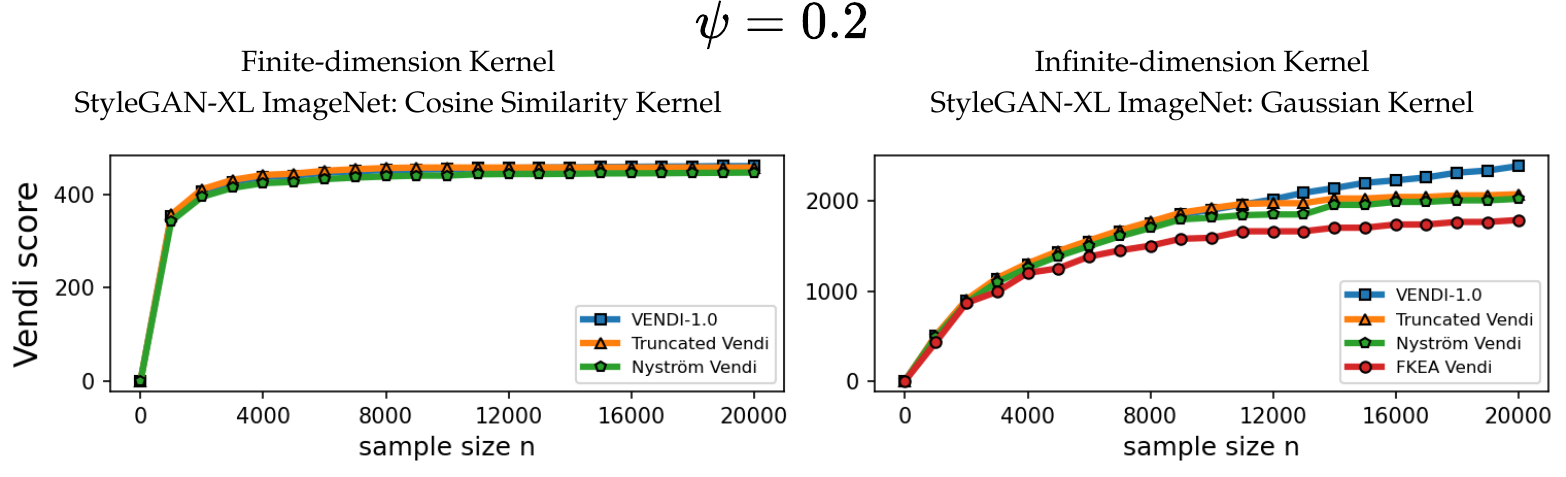}}
    \subfigure{\includegraphics[width=0.7\textwidth]{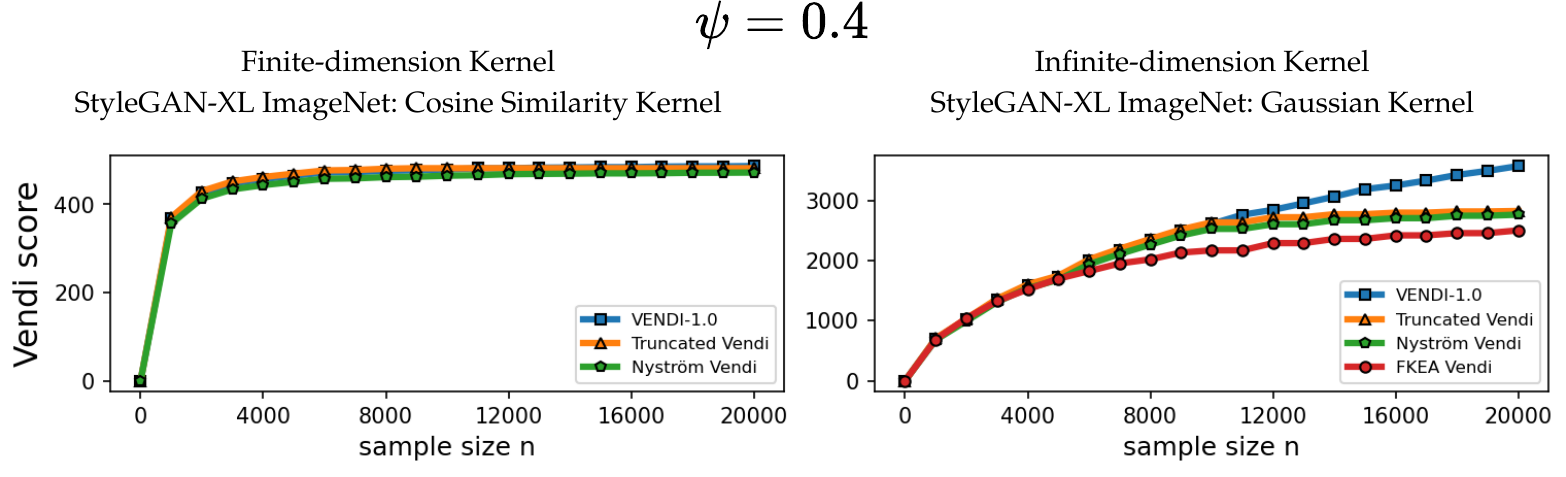}}
    \subfigure{\includegraphics[width=0.7\textwidth]{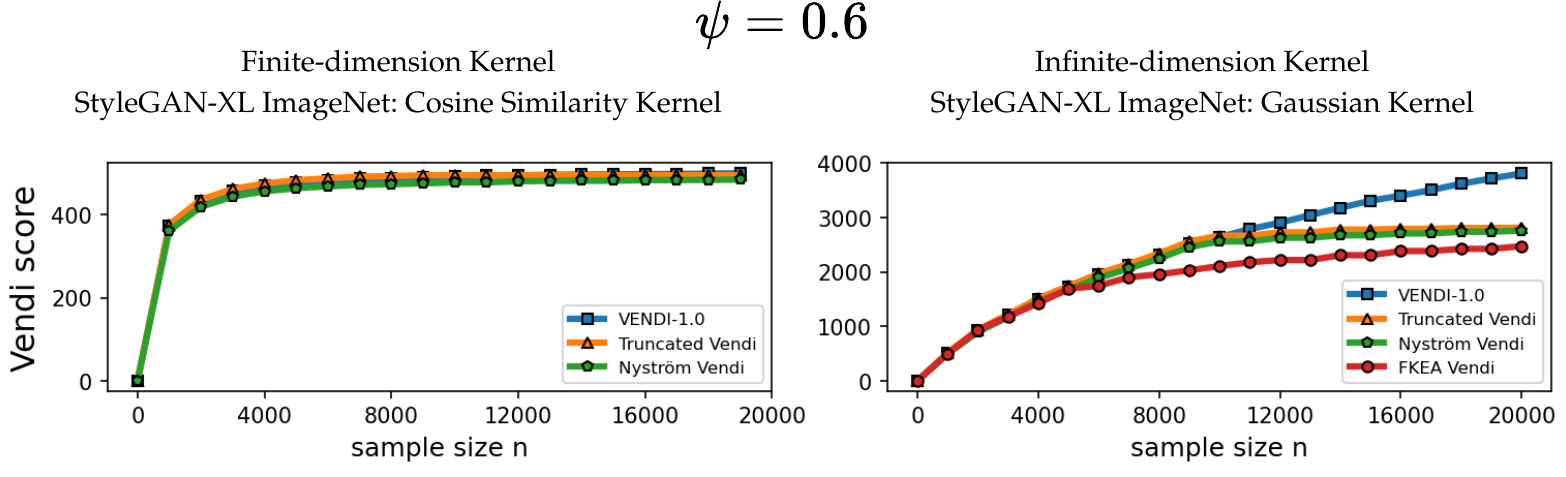}}
    \subfigure{\includegraphics[width=0.7\textwidth]{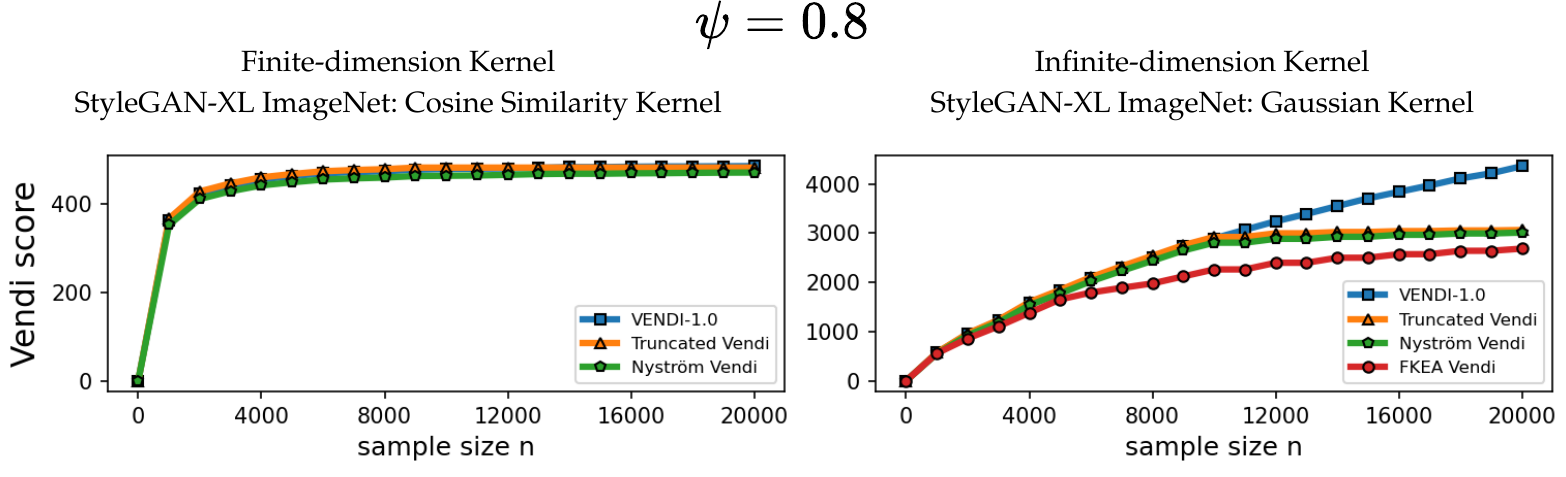}}
    \subfigure{\includegraphics[width=0.7\textwidth]{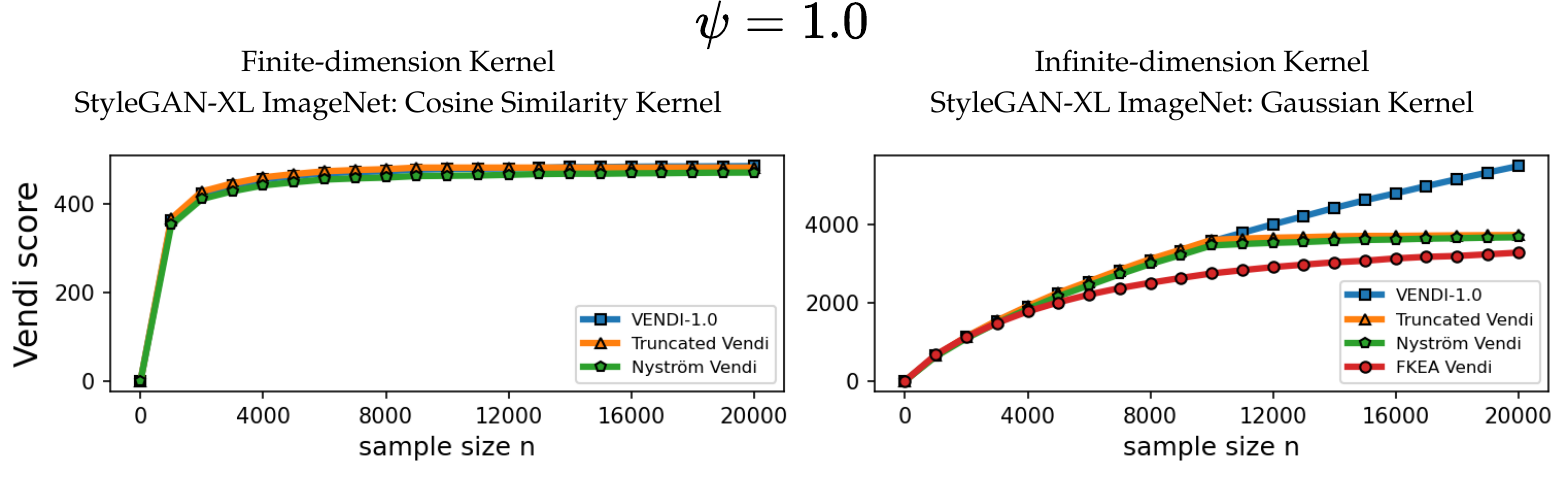}}
    \caption{Statistical convergence of Vendi score for different sample sizes on StyleGAN-XL generated ImageNet data: (Left plot) finite-dimension cosine similarity kernel (Right plot) infinite dimension Gaussian kernel with bandwidth $\sigma=40$. \emph{DinoV2} embedding (dimension 768) is used in computing the scores.}
  \label{VENDI_styleganxl_convergence}
\end{figure}

\begin{table}
    \centering
    \caption{Statistical convergence of diversity scores for different sample size  on DALL-E 3 generated MSCOCO data}
    \begin{tabular}{lccccccc}
    \toprule
    $n$ & VENDI-1.0 & RKE & Vendi-t & FKEA-Vendi & Nystrom-Vendi & Recall & Coverage \\
    \midrule
    2000 & 239.91 & 13.47 & 239.91 & 228.69 & 239.91 & 0.76 & 0.86 \\
    4000 & 315.35 & 13.51 & 315.35 & 280.68 & 315.35 & 0.81 & 0.87 \\
    6000 & 357.15 & 13.56 & 346.27 & 310.9 & 345.49 & 0.83 & 0.91 \\
    8000 & 392.36 & 13.56 & 354.8 & 329.56 & 357.41 & 0.87 & 0.91 \\
    \bottomrule
    \end{tabular}
    \label{tab:dalle3-mscoco}
\end{table}

\begin{table}
    \centering
    \caption{Statistical convergence of diversity scores for different sample size on SDXL generated MSCOCO data}
    \begin{tabular}{lccccccc}
    \toprule
    $n$ & VENDI-1.0 & RKE & Vendi-t & FKEA-Vendi & Nystrom-Vendi & Recall & Coverage \\
    \midrule
    2000 & 187.17 & 10.65 & 187.17 & 173.06 & 187.18 & 0.78 & 0.85\\
    4000 & 236.49 & 10.7 & 236.49 & 222.78 & 236.08 & 0.82 & 0.87\\
    6000 & 264.82 & 10.7 & 258.21 & 236.37 & 257.34 & 0.86 & 0.87\\
    8000 & 289.08 & 10.71 & 265.84 & 251.59 & 266.23 & 0.86 & 0.86\\
    10000 & 304.44 & 10.72 & 267.39 & 256.24 & 268.34 & 0.86 & 0.87\\
    \bottomrule
    \end{tabular}
    \label{tab:sdxl-mscoco}
\end{table}

\begin{table}
    \centering
    \caption{Compilation time (in seconds) of different Vendi scores with increasing sample size}
    \begin{tabular}{lccccccc}
    \toprule

    \multirow{2}{*}{Metric} & \multicolumn{7}{c}{samples $n$} \\	
     & 10000 & 20000 & 30000 & 40000 & 50000 & 60000 & 70000 \\
    \midrule
    Vendi & 97s & 631s & 1868s & - & - & - & - \\
    FKEA-Vendi & 19s & 36s & 53s & 71s & 88s & 105s & 124s \\
    Nystrom-Vendi & 31s & 44s & 78s & 91s & 112s & 136s & 164s \\

    \bottomrule
    \end{tabular}
    \label{tab:time complexity}
\end{table}

\vspace{-2mm}
\subsection{Bandwidth $\sigma$ Selection}
In our experiments, we select the Gaussian kernel bandwidth, $\sigma$, to ensure that the Vendi metric effectively distinguishes the inherent modes within the dataset. The kernel bandwidth directly controls the sensitivity of the metric to the underlying data clusters. As illustrated in Figure \ref{bandwidth_illustration}, varying $\sigma$ significantly impacts the diversity computation on the ImageNet dataset. A smaller bandwidth (e.g., $\sigma = 20, 30$) results in the metric treating redundant samples as distinct modes, artificially inflating the number of clusters, which in turn slows down the convergence of the metric. On the other hand, large bandwidth results in instant convergence of the metric, i.e. in $\sigma=60$ $n=100$ and $n=1000$ have almost the same amount of diversity.

\section{Selection of Embedding space}
To show that proposed truncated Vendi score remains feasible under arbitrary embedding selection, we conducted experiments from Figures \ref{VENDI_ffhq_truncation} and \ref{VENDI_imagenet_truncation}. Figures \ref{VENDI_imagenet_clip_diversity}, \ref{VENDI_ffhq_truncation_clip_diversity}, \ref{VENDI_imagenet_swav_diversity} and \ref{VENDI_ffhq_truncation_swav_diversity} extend the results to CLIP \cite{radford_learning_2021} and SWaV \cite{caron_unsupervised_2020} embeddings. These experiments demonstrate that FKEA, Nyström and $t$-truncated Vendi correlate with increasing diversity of the evaluated dataset. We emphasize that proposed statistic remains feasible under arbitrary embedding space that is capable of mapping image samples into a latent space.

\begin{figure}
    \centering
    \includegraphics[width=\textwidth]{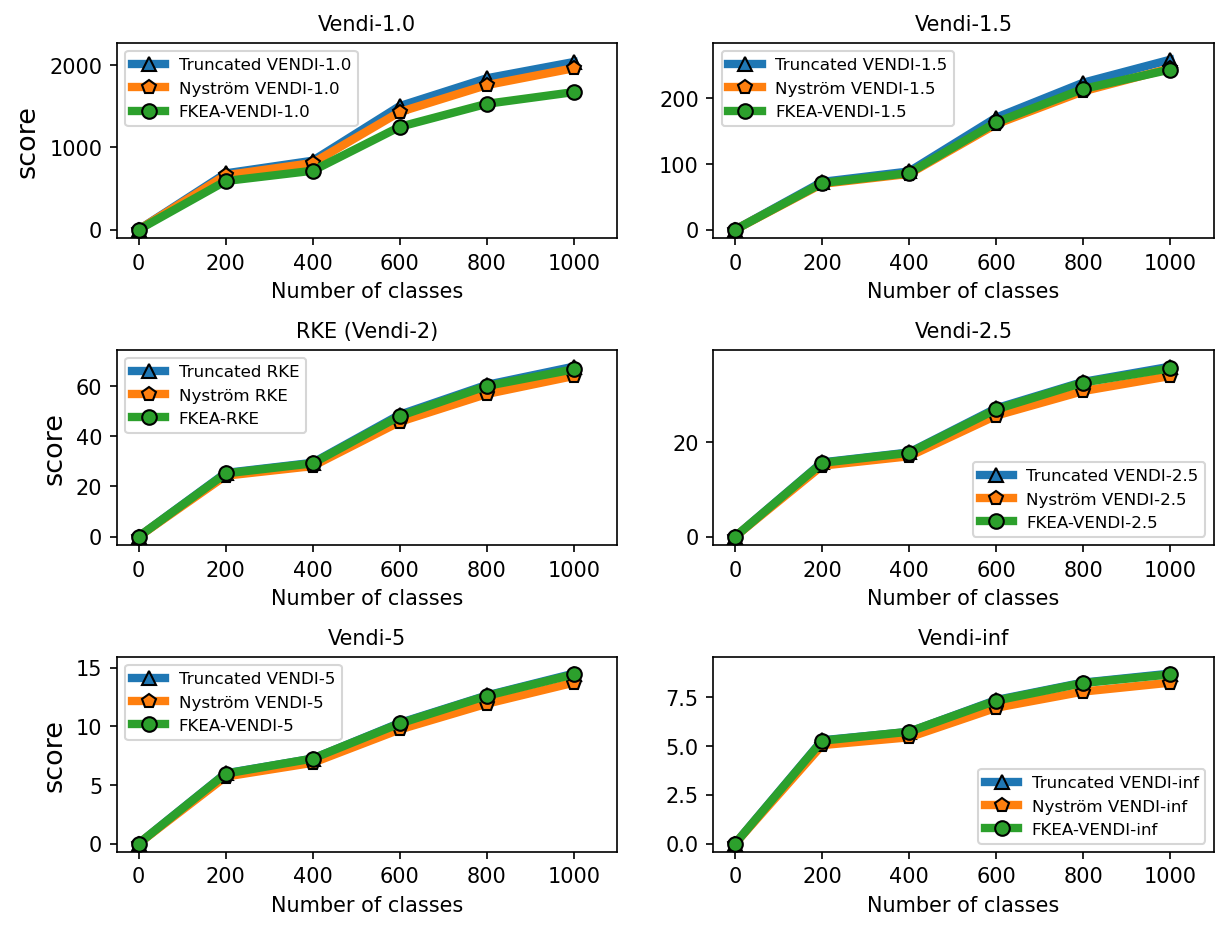}
    \caption{Diversity evaluation of Vendi scores on ImageNet dataset with varying number of classes based on \textit{CLIP} embedding and bandwidth $\sigma=5.0$}
  \label{VENDI_imagenet_clip_diversity}
\end{figure}

\begin{figure}
    \centering
    \includegraphics[width=\textwidth]{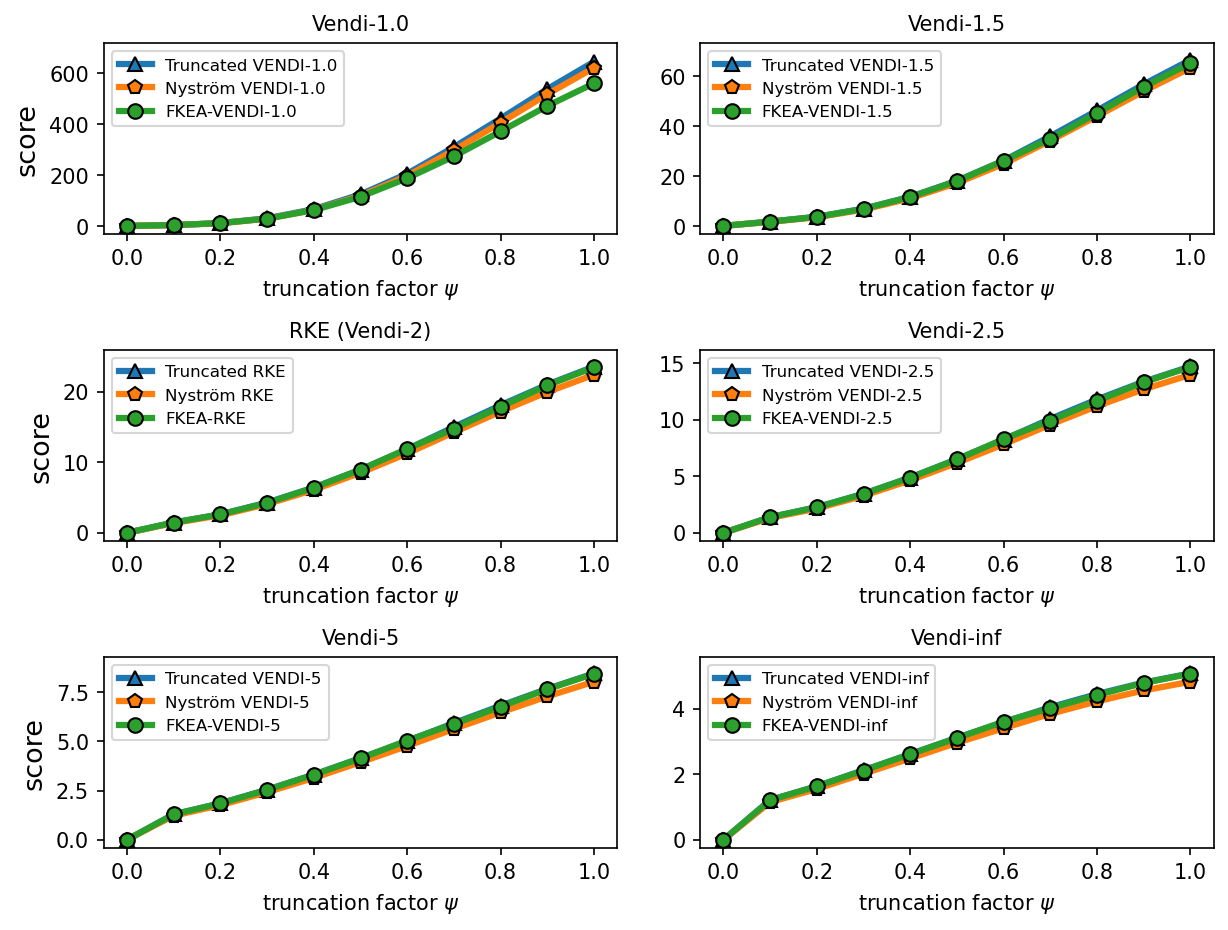}
    \caption{Diversity evaluation of Vendi scores on truncated StyleGAN3 generated FFHQ with varying truncation coefficient $\psi$ based on \textit{CLIP} embedding and bandwidth $\sigma=5.0$}
  \label{VENDI_ffhq_truncation_clip_diversity}
\end{figure}

\begin{figure}
    \centering
    \includegraphics[width=\textwidth]{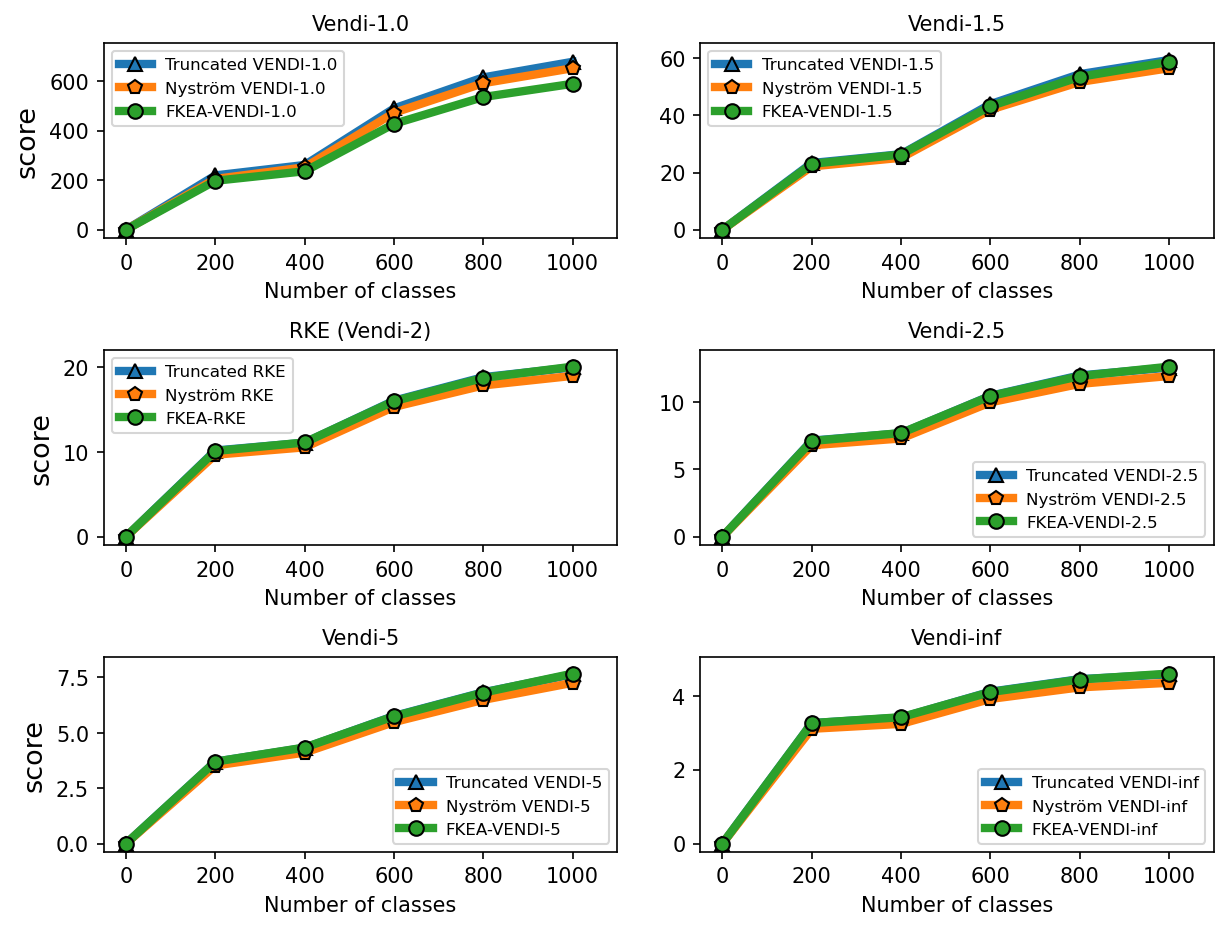}
    \caption{Diversity evaluation of Vendi scores on ImageNet dataset with varying number of classes based on \textit{SWaV} embedding and bandwidth $\sigma=1.0$}
  \label{VENDI_imagenet_swav_diversity}
\end{figure}

\begin{figure}
    \centering
    \includegraphics[width=\textwidth]{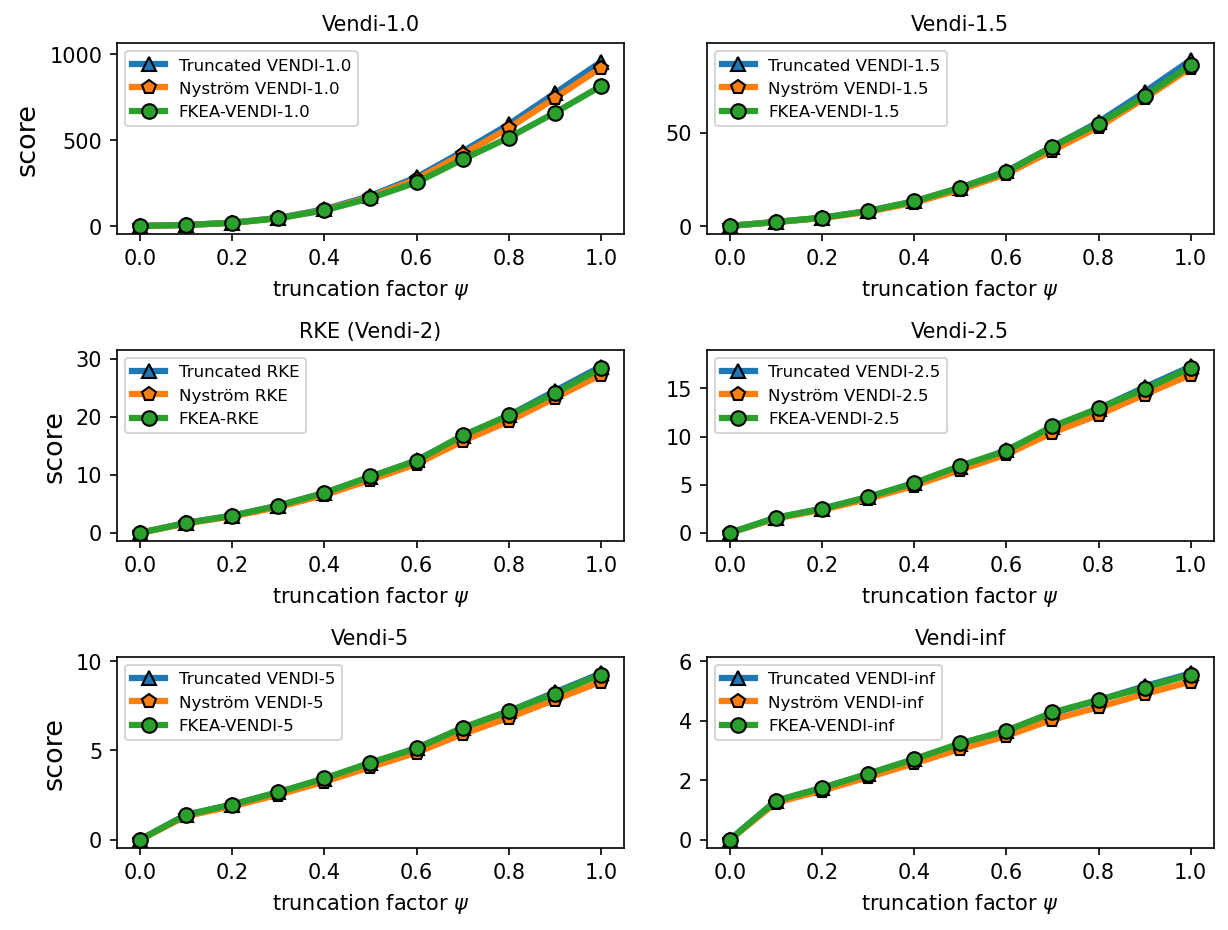}
    \caption{Diversity evaluation of Vendi scores on truncated StyleGAN3 generated FFHQ with varying truncation coefficient $\psi$ based on \textit{SwAV} embedding and bandwidth $\sigma=1.0$}
  \label{VENDI_ffhq_truncation_swav_diversity}
\end{figure}

\end{appendices}

\end{document}